\documentclass[11pt]{article}

\newcommand{\shorttitle}{Co-Recursive Semantic Self-Extension}

\usepackage[margin=1in]{geometry}

\usepackage{amsmath,amssymb,amsthm,mathtools}

\usepackage{booktabs}

\usepackage{graphicx}
\graphicspath{{figures/}{./}}

\usepackage[hidelinks]{hyperref}
\usepackage{xcolor}

\usepackage{fancyhdr}
\usepackage{natbib}

\theoremstyle{definition}
\newtheorem{definition}{Definition}

\newtheorem{assumption}{Assumption}
\theoremstyle{plain}
\newtheorem{proposition}{Proposition}
\newtheorem{theorem}{Theorem}
\newtheorem{corollary}{Corollary}

\theoremstyle{remark}
\newtheorem*{remark}{Remark}

\title{Co-Recursive Semantic Self-Extension: A Mathematical Foundation for Open-Ended Intelligence}

\author{%
  Simon Hegelich\textsuperscript{1}\thanks{Corresponding author: simon.hegelich@gaialogic.ai}, 
  Bardh Prenkaj\textsuperscript{1}, 
  Habiba Sarhan\textsuperscript{1}, 
  Morteza Shahrezaye\textsuperscript{1}, 
  Kolja Hegelich\textsuperscript{1},\\
  Erik Cambria\textsuperscript{2}, 
  Christian Montag\textsuperscript{3}%
}

\date{}

\begin{document}
\maketitle
\footnotetext[1]{GaiaLogic Unipessoal Lda.}
\footnotetext[2]{Nanyang Technological University, Singapore}
\footnotetext[3]{University of Macau}
\setcounter{tocdepth}{2}
\pagenumbering{roman}
\tableofcontents
\clearpage
\pagenumbering{arabic}

\section*{Abstract}
\addcontentsline{toc}{section}{Abstract}

Artificial intelligence has advanced through scale, data, and computation, yet remains bounded by its representational framework. We ask \textit{what makes intelligence open-ended}.

Every inferential system selects meaningful distinctions, limiting what it can learn. Within a fixed purpose, performance reaches a ceiling set by its semantic domain. Teleological succession, in which purposes arise from earlier ones, faces a similar limit: Boolean-closed purpose-transformations exhaust their reachable semantic territory.

We shift from purposes to inferential systems. Drawing on Kant and Gödel, we model cognitive states as partial evaluations over a formal language. Unresolved propositions may be undecided or undecidable. Reasoning can resolve the former but not the latter, leaving a frontier for extension. Non-trivial reductions of undecidable propositions expose undecidable intermediate premises whose reflective reification supports multiple decisive extensions. Finite trace-guided generation remains fallible, requiring meta-level selection and polycontextural exploration. Using Kolmogorov complexity, we show that no architecture with a fixed state description, transition mechanism, and initial condition can generate histories incorporating irreducible algorithmic information.

Because physical systems are fixed in this sense, we examine collectives. Coupling finitely many finite agents does not remove the bound. Extension requires persistent external memory preserving semantic content and historical coordinates. Collective generative capacity depends on balancing overlap for transfer with difference for generation.

We propose semantic self-extension as a design principle: \textit{intelligence should not merely resolve ambiguity within existing meanings, but transform unresolved structures into new spaces of meaning}, generating new distinctions, questions, representations, and purposes.

\section{Introduction: Epistemological Heuristics and Logical Design in AGI Architecture}\label{sec:introduction}

Artificial General Intelligence (AGI) research is not stuck for lack of computational power or data. It is stuck because of the architectural paradigms everyone keeps reaching for. Building a system capable of genuine machine autonomy calls for something older than deep learning: the theoretical frameworks of Immanuel Kant and Kurt Gödel, used here as engineering heuristics rather than as mathematical laws the machine must obey. Kant, in the \textit{Critique of Pure Reason} \citep{kant1998}, argued that human cognition actively constructs experience through \textit{a priori} structures rather than passively absorbing it. Take that seriously as a design principle, and the implication is direct. Without some pre-existing framework for organizing raw data, a system stays ``blind,'' unable to form coherent, self-directed goals. Gödel's Incompleteness Theorems \citep{godel1931} pose a parallel engineering challenge. Any formal system expressive enough to be useful will contain propositions it cannot decide. We do not need Gödel's theorems to hold in some strict mathematical sense here; what we need is the heuristic they suggest: a system that can invoke a higher-level observer language whenever it hits local undecidability, rather than simply failing.

Modern autoregressive transformers, the large language models behind most current AI, illustrate what happens when this Kantian principle gets ignored. The prevailing ``scaling hypothesis'' holds that simply growing these models and their training data will eventually produce AGI on its own. But an LLM is fundamentally a stateless function: it maps an input to a probability distribution over tokens, based on statistical correlations.
Contemporary transformers do use architectural priors, such as attention, positional encodings, and specific tokenization schemes, to learn implicit hierarchical and syntactic structure. Even so, prediction is all they are optimized for; there is no self-directed goal-generating mechanism underneath it. Lacking any pre-experiential drive to act or evaluate, an LLM depends entirely on the prompt it is given. It is a remarkably sophisticated mirror of human language, but without an \textit{a priori} foundation of its own, it has no way to generate its own motivation or direction.

Good Old-Fashioned AI (GOFAI) took the opposite approach: building intelligence top-down from rigid, pre-defined logical rules. It ran into the Gödelian problem from the other direction. Committed to a strict version of the Physical Symbol System Hypothesis, GOFAI systems aimed for total logical determinism, and when they hit an undecidable proposition, they failed: acting as partial functions that crash on undefined input. Nothing in the architecture treated ``undecidable'' or ``ambiguous'' as a legitimate, computable state rather than an error, because there was no meta-level built to operate on anything but two-valued logic. So when the real world's ambiguity inevitably showed up, these systems had no way to reflect on their own rules and adapt. They just got brittle and stopped.

Neuro-symbolic AI, which bolts a neural network onto a symbolic logic engine, is often pitched as a bridge between these two failure modes. We argue it is not one. Rather than producing a genuinely new kind of cognition, it tends to combine the weaknesses of both parents: a neural component that is still a blind, reactive statistical engine with no intrinsic direction, wired to a symbolic component that is still a rigid, flat system with no capacity to reflect on the neural network's ambiguities.

Put these historical failures together, and a pattern emerges: AGI is unlikely to come from any single, monolithic model. No one paradigm has managed to combine the Kantian requirement for \textit{a priori} structure with a Gödelian meta-level for handling incompleteness, so the question worth asking shifts from which model to which system. Minsky essentially made this argument in \textit{The Society of Mind} \citep{minsky1986}: \textit{intelligence is not a single algorithm; it is an emergent property of many different sub-agencies interacting}. Getting past the dead ends above means designing a composite architecture that keeps these requirements structurally separate while allowing them to integrate.

A prominent recent effort to move beyond purely statistical models is the $(C, U, V)$ framework of \citet{he2026cuv}, which characterizes an AGI agent as a point in a joint space comprising a cognitive architecture $C$, a set of potential or skill functions $U$, and a set of value functions $V$. Observing that conversational models such as ChatGPT are trained on corpora that ``inherently embed numerous---often conflicting---value systems,'' \citet[p.~179]{he2026cuv} propose a value-driven architecture in which $C$ selects active dimensions from $V$, recruits skills from $U$ to interpret scenes as relational parse graphs, and plans actions that increase the selected values. Yet their formalization also exposes the boundary of value-driven design within a pre-structured space: as $(C, U, V)$ co-evolves, the process approaches an equilibrium at which the architecture can no longer motivate or orchestrate new learning, a limit they term ``the `stopping problem' of Tong Agents, based on the analogy of the halting problem in a Turing machine'' \citep[p.~166]{he2026cuv}. Overcoming that stopping limit requires two architectural principles developed in the sections that follow: treating undecidable or conflicting evaluations not as defects to be removed by ``predefining clear values and principles'' \citep[p.~179]{he2026cuv} but as the source of \emph{semantic surplus}, and equipping the architecture with an \textit{a priori} meta-teleology whose explicit directive is to construct new teleological dimensions whenever existing purposes saturate.

\section{Laws of Intelligent Systems}\label{sec:laws}

The central claim of this paper is that intelligence is not the consequence of scale alone, but of a particular organizational structure. Two independent philosophical results constrain this structure. Kant establishes that every cognitive system requires an \textit{a priori} organization that precedes experience. Gödel demonstrates that every sufficiently expressive formal system is necessarily incomplete. Together, they imply a surprisingly narrow class of architectures capable of open-ended intelligence. \textbf{The following laws are our design interpretations of how we think an open-ended intelligent system could evolve from our basic assumptions.}

The first law is a strict duality: the system needs an object-level structure and a distinct meta-level structure, kept separate. The object level is responsible for interacting with the world. It conducts observations, performs operations, generates hypotheses, and stores empirical knowledge.

The meta level acts on the object level itself rather than on the world directly. It evaluates reasoning, compares alternatives, transforms representations, introduces new abstractions, and reorganizes prior knowledge.

Neither level alone is sufficient. An object system without a meta-system can accumulate experience but cannot systematically improve its own reasoning. A meta-system without an object system possesses nothing to reflect upon.

Following Kant, this duality is therefore not optional but constitutes the minimum a priori organization required for cognition. For both levels to interact, they must share a common representational substrate. Every internal state, rule, observation, and transformation must therefore be represented in a similar representational space. The architecture is consequently homoiconic: reasoning itself becomes manipulable as data.

Only under this condition can genuine self-reference emerge.

The second law follows from Gödel's limits on that shared language \citep{godel1931}: i.e., any sufficiently expressive system necessarily contains propositions that cannot be decided within its own axioms. For intelligent systems, this result has a direct operational interpretation. Incomplete knowledge is not an implementation defect; it is the driving force behind development, making it a permanent structural property. Instead of viewing incompleteness as failure, we reinterpret it as ambiguity. Ambiguity denotes every point at which multiple internally consistent interpretations remain simultaneously possible because the current representational system lacks sufficient expressive power to distinguish among them. Every intelligent system, therefore, contains a permanently evolving frontier of ambiguity. This frontier can move, but it cannot disappear.

That openness leads to the third law, and the most important one: structural ambiguity necessarily generates an excess of meaning. The system is thus incomplete: whenever multiple interpretations remain simultaneously admissible, it possesses more semantic possibilities than its current axioms can determine. We call this excess semantic surplus. Semantic surplus is not randomness. It is the collection of internally coherent but mutually competing explanations that remain unresolved inside the current representational framework. A purely static reading of Gödel's theorem establishes only a negative boundary---the existence of an unprovable sentence---because, as \citet[p.~10]{gunther1979life} observes, Gödelian arithmetization ``concerns only the extensional domain of logic and bypasses those intensional relations where dialectical principles come into play.'' Yet a non-omniscient system cannot know in advance whether a proposition is provable before attempting to derive it. The generative excess therefore does not reside in the unprovable sentence in isolation, but in the movement of the proof attempt itself: in Hegel's formulation, ``the real issue is not exhausted by stating it as an aim, but by carrying it out, nor is the result the actual whole, but rather the result together with the process through which it came about'' \citep[p.~2]{hegel1977phenomenology}. When a derivation stalls against an undecidable frontier, that operational failure does not resolve ``into abstract nothingness, but essentially only into the negation of its particular content'' \citep[p.~33]{hegel2010}; it is a \emph{determinate negation} through which ``a new form has thereby immediately arisen, and in the negation the transition is made'' \citep[p.~51]{hegel1977phenomenology}. Because in reflective logic ``content is at the same time the negative reflection of the form determinations into themselves'' \citep[p.~396]{hegel2010}, the incomplete derivation---which at the object level is merely the failed \emph{form} of a proof---becomes, when reflected upon by the meta-level, a determinate \emph{content} that defines competing ways of extending the system. We argue that genuine intelligence depends on maintaining and exploiting it. Semantic surplus, therefore, becomes the primary computational resource of intelligent systems. Without surplus, there exists nothing to evaluate, transform, or discover. Semantic surplus is the raw material autonomy runs on: it keeps the system from ever settling into a single deterministic loop or a static picture of reality. Because the available interpretations always outnumber what the \textit{a priori} rules can decide among, the system carries a permanent internal tension, the gap between an incomplete rule set and a surplus, and it is exactly this gap that the meta-structure has to navigate.

With a shared language and this tension between incompleteness and surplus in place, the system's operational dynamics follow fairly directly. A common linguistic topology makes basic feedback loops possible, in which one sub-agency's output becomes another's input, and that interoperability, in turn, allows architectural stacking: higher-order systems built on top of two or more foundational ones, which is more or less Minsky's society of mind realized concretely. Loops and stacking add complexity, but they are not what drives autonomy. The recursive structure is therefore not introduced for computational convenience. It is forced by structural incompleteness. Whenever ambiguity cannot be resolved within the current representational system, the system must construct a successor interpretative context that can reason about its predecessor.\\
Our fourth formulated law, therefore, is: an open-ended intelligence system is based on a recursive structure. Recursion here works in two complementary directions. The first, inward recursion, follows directly from the shared representational space: because the meta-structure speaks the same language as the object level, it can treat the object level's operations as data to manipulate. When the object-level produces an ambiguous surplus, the meta-system reads those unresolved results and feeds them back in as new starting points. More significantly, inward recursion also allows the meta-structure to rewrite the object-level's own rules and axioms, adjusting its constraints to better handle the surplus it generates. The most rigorous formalization of such homoiconic inward recursion in the literature is Schmidhuber's Gödel machine \citep{schmidhuber2006goedel}, a self-referential problem solver that ``rewrites any part of its own code as soon as it has found a proof that the rewrite is useful, where the problem-dependent utility function and the hardware and the entire initial code are described by axioms encoded in an initial proof searcher which is also part of the initial code'' \citep[p.~1]{schmidhuber2006goedel}. Yet because every self-modification must be certified by a target theorem derived from those initial axioms, the Gödel machine remains bound to the deductive closure of its starting formalism. Even when it modifies its own axioms or utility function, it may only replace axioms with ``theorems derivable from the original axioms'' \citep[p.~17]{schmidhuber2006goedel} and alter its utility function ``only if their new values are provably better according to the old ones'' \citep[p.~17]{schmidhuber2006goedel}, so that ``all possible `meta-levels' of the self-referential system are collapsed into one'' \citep[p.~13]{schmidhuber2006goedel}.

Standard inward recursion of this kind, which reduces every self-modification back to a base case certified by the initial axioms, is not enough on its own for open-ended intelligence; left alone, it settles into the deductive span of its initial configuration. So the architecture also needs co-recursion. Instead of collapsing meta-levels into a fixed state space, a co-recursively self-modifying system generates a sequence of distinct, higher-order semantic frameworks, each an escalation from the last. The permanent, unresolvable tension from its overcomplete surplus is what forces this: the system resolves its ambiguities by building a successor manifold with more expressive dimensions than the one it started in. That is what keeps it from ever halting. It keeps generating a new, elevated version of its own state space, turning what it has already synthesized into the coordinate system for its next stage. Dropping deductive certification within the initial axiom system does not, however, reduce self-modification to arbitrary or unguided mutation. Instead of requiring a proof inside the old regime that the rewrite improves a fixed utility function, a co-recursive step is disciplined from below by \emph{reflective preservation}---the requirement that the validated inferential structure of the predecessor remain representable within the successor via a structure-preserving translation---and from above by \emph{decisive resolution}, the requirement that the successor framework settle at least one frontier proposition that was formally undecidable in its predecessor. Because multiple candidate extensions can satisfy both constraints simultaneously without any deduction inside the predecessor deciding among them, recursion produces structured exploration rather than deductive certainty: a higher-order interpretation may fail to resolve the ambiguity that motivated it or may introduce fresh inconsistencies of its own.

Because co-recursive exploration generates multiple candidate interpretations and successor frameworks that cannot be collapsed into a single two-valued deductive system, the architecture requires a fifth law: polycontexturality.\footnote{The concept of \emph{polycontexturality}, originally developed in philosophy and logic by Gotthard Günther \citep{gunther1979life} and subsequently adopted in social systems theory by Niklas Luhmann \citep{luhmann1999gesellschaft}, holds that a phenomenon can be perceived or interpreted across multiple, mutually irreconcilable frameworks simultaneously, operating through second-order observation in which an observer recognizes that its own first-order distinctions are not the only admissible ones \citep{bartmanski2021imbricated}.} In a polycontextural architecture, distinct logical domains---each locally coherent and governed by its own inferential commitments---coexist without being prematurely forced into a single global valuation. As \citet{gunther1979life} shows, while a single contexture is a domain whose internal structure ``can be exhaustively described by a two-valued logic'' \citep[p.~3]{gunther1979life} and in which ``every logical operation we can perform is confined to the contexturality in which it originates'' \citep[p.~3]{gunther1979life}, ``{[}t{]}he relation of the cognizing subject to its range of objects is always one of discontexturality'' \citep[p.~4]{gunther1979life}, making it ``impossible to bring two different contexturalities into an immediate confrontation'' \citep[p.~4]{gunther1979life} as if they shared ``a common ontic measure'' \citep[p.~13]{gunther1979life}. Although each individual regime remains locally two-valued, ``the positive value in contexture A'' is not identical with ``the positive value in contexture B'' \citep[p.~6]{gunther1979life}, so that ``linking two elementary contextures into a compound contexture'' requires a third contexture to ``mediate between the other two'' \citep[p.~15]{gunther1979life}. Architecturally, this means that the object/meta duality of the First Law and the competing candidate extensions of the Fourth Law cannot be flattened into one mono-contextural calculus: where classical two-valued logic treats incompatibility between contexts as a contradiction that collapses the system by explosion, the meta-level serves as the mediating contexture in which partially incompatible regimes are held open as a structural locus for forming new distinctions.

We can formulate a sixth law from the structural independence of recursive contexts. Because recursive transformations may fail, no intelligent system can discard previous semantic structures. Earlier interpretations may remain valuable in future contexts. Memory, therefore, becomes a logical necessity rather than a storage convenience.

However, preserving semantic content alone is insufficient. The meaning of every representation depends upon the recursive context in which it was produced. Memory must therefore preserve not only what was represented but also where it resides within the system's evolving semantic organization. Every remembered structure consequently retains both its semantic content and its position within the recursive network of transformations from which it emerged. Only under these conditions can future recursive processes correctly revisit, reinterpret, compare, or reactivate earlier reasoning.

Intelligence, therefore, depends not only on generating semantic surplus but also on preserving the semantic coordinates that make this surplus accessible. The architecture accumulates not merely a history of representations but a navigable semantic landscape in which every memory remains embedded within its original context while simultaneously becoming available for future reinterpretation.

Recursion, memory, and polycontexturality generate an expanding space of possible future transformations. A principle is therefore required that determines which transformations should be pursued. Such a principle cannot itself arise from empirical processing, because empirical processing already presupposes a criterion by which alternative transformations are evaluated. Following Kant, this criterion must therefore be \textit{a priori}. We identify this organizing principle as teleology, which constitutes our seventh and final law. Rather than being a static logical axiom, teleology provides the architecture with a built-in ``why'' that precedes empirical processing and guides the system's recursive evolution.

The particular content of this teleology is contingent. Different intelligent systems may begin with different ultimate purposes, such as preserving life, maximizing knowledge, promoting human welfare, or pursuing entirely different objectives. What is necessary, however, is not the specific content of the teleology but its recursive accessibility. Because the architecture is homoiconic, teleology itself is represented within the common symbolic substrate and therefore becomes an object of meta-reasoning alongside every other internal state. The system must consequently be capable of representing, evaluating, transforming, and extending its own goals.

The contrast with the $(C, U, V)$ formulation of \citet{he2026cuv} illustrates why a fixed value vector cannot sustain open-endedness, even when paired with an adaptive cognitive architecture. In their model, the $V$ system is ``a high-dimensional vector, with each dimension corresponding to a specific human need, such as health, wealth, and peace'' \citep[p.~167]{he2026cuv}, while the cognitive architecture $C$ applies a parameterized selector $V_{\mathrm{select}}(V, \theta_{\mathrm{agent}})$ ``with each element indicating whether a particular value dimension has been activated'' \citep[p.~170]{he2026cuv}. Reweighting or sequentially activating coordinates within a finite, pre-indexed vector space is not a transformation of teleology. As \citet{he2026cuv} themselves note in their analysis of value-space convergence, when an agent's value system is constrained by its architecture, ``it may not recognize the purpose or benefits of continued learning, which can also lead to a convergence in learning'' \citep[p.~175]{he2026cuv}---so that once basic needs are met, agents ``stop striving for self-improvement, essentially `lying flat'\,'' \citep[p.~175]{he2026cuv}. They illustrate the same limit at the population level: a monkey population possesses in its skill space $U$ ``the physical ability to perform actions such as striking stone or wood together, which would eventually produce fire. However, the population's $V$ space lacks the ability to understand the fact that these actions might lead to fire or to recognize the value of starting a fire by friction'' \citep[p.~178]{he2026cuv}. Within the architecture proposed here, escaping that trap requires two conditions that a selector over fixed $V$-dimensions lacks. First, by the first and second laws, $C$, $U$, and $V$ must share a homoiconic representational substrate; without it, as \citet{he2026cuv} acknowledge when asking whether an agent can construct another agent, the system ``has a structured cognitive framework but may lack the ability to fully understand its own structure without learning'' \citep[pp.~178--179]{he2026cuv}. Second, the architecture cannot start merely with first-order object values; its foundational \textit{a priori} teleology must include the meta-directive to seek and construct new telos out of semantic surplus whenever existing value dimensions saturate or conflict.

This distinction is fundamental. A recursive architecture with a fixed teleology remains constrained by that teleology, regardless of how sophisticated its reasoning becomes. While it may recursively optimize its behavior, it cannot recursively transform the principle that governs its optimization. Such a system, therefore, remains bounded by its original \textit{a priori} organization. An open-ended intelligent system, by contrast, must be capable of recursively transforming its own teleological structure. The recursive closure of the architecture is achieved only when its goals themselves become subject to recursive evaluation and transformation. In this sense, open-ended intelligence is grounded not in any particular final objective, but in the capacity to continually construct, evaluate, and refine the very teleologies that guide its future development.

The recursive transformation of teleology does not imply unbounded expansion within a single semantic framework. On the contrary, each recursive step constructs a new representational framework that is locally coherent and internally stable while extending the expressive power of its predecessor. Open-endedness therefore emerges not from infinite growth inside one semantic system, but from an indefinite succession of finite, locally stable semantic structures. Each recursive transformation resolves part of the current ambiguity by establishing a new temporary equilibrium. In contrast, the remaining unresolved ambiguity provides the semantic surplus that eventually motivates the construction of a successor framework. In this way, intelligence progresses through a sequence of recursive semantic reorganizations rather than through unlimited expansion within a fixed representational space. Having established the seven necessary laws governing intelligent systems, we can now formulate a precise definition of open-ended intelligence. Importantly, this definition is not assumed a priori but follows directly from the preceding derivation.

\begin{definition}[Open-Ended Intelligence]
    \textit{An open-ended intelligent system is a recursively self-organizing cognitive architecture that continuously resolves structural ambiguity by generating semantic surplus, preserving alternative semantic structures in memory, and recursively transforming both its representational framework and its own teleology. Each recursive transformation increases the system's expressive capacity while simultaneously introducing new ambiguity, ensuring that cognitive development remains indefinitely extensible.}
\end{definition}

The seven laws should not be understood as independent design principles but as a single recursive process. Each law establishes the conditions required for the next. In contrast, the final law---recursive teleology---initiates a new object/meta-organization, thereby reconstructing the first law at a higher semantic level. The architecture is therefore recursively closed: applying the seven laws to an intelligent system necessarily produces another intelligent system governed by the same seven laws, but operating within a more expressive semantic framework. Open-ended intelligence emerges not from any individual law in isolation, but from the indefinite recursive application of the complete system of laws to itself.

\section{Why Open-Ended Intelligence Requires Recursive Teleology (Design Rationale)}

\subsection{The Scaling Paradox}\label{sec:scaling_paradox}

The dominant paradigm in contemporary artificial intelligence operates under the ``scaling hypothesis''. This hypothesis assumes that simply increasing parameters, computational steps, and data volume within statistical architectures will eventually yield AGI. Yet, this technical expansion fails to produce or explain the defining attributes of general intelligence: autonomous curiosity, self-directed goal formation, real-time ontology creation, or scientific revolutions. Under pure scaling, models converge toward the empirical probability function of their training distributions, leaving them entirely dependent on human curation or reinforcement feedback. This ``parasitic intelligence'' is structurally incapable of true open-ended novelty.

\subsection{Every Intelligence Begins with Discrimination}\label{sec:intelligence_begins}

A fundamental ontological error of the scaling paradigm is the conflation of raw
signals with meaning. Raw computer-readable signals $S$ possess no inherent
meaning; they become semantics, or data, only in relation to an active purpose. Let $\Phi$ be a space of purposes and let $M \subseteq S \times \Phi$ be a relevance relation, writing $M(s,\phi)$ for ``$s$ is meaningful under $\phi$''. Each purpose induces a semantic domain
\begin{equation}\label{eq:dphi}
  D_{\phi} \;:=\; \{\, s \in S \mid M(s,\phi) \,\}, \qquad \phi \in \Phi .
\end{equation}
An agent does not stand in the same relation to every purpose. At any stage $n$ of its development, it pursues some purposes, has considered and rejected others, and has not yet considered the rest. We therefore partition its current repertoire $\Phi_{n} \subseteq \Phi$ into three blocks,
\begin{equation}\label{eq:partition}
  \Phi_{n} \;=\; \Phi^{+}_{n} \,\dot\cup\, \Phi^{-}_{n} \,\dot\cup\,
                 \Phi^{\circ}_{n},
  \qquad
  \Phi^{+}_{n} \neq \emptyset, \quad \Phi^{-}_{n} \neq \emptyset,
\end{equation}
where $\Phi^{+}_{n}$ are the \emph{pursued} purposes, $\Phi^{-}_{n}$ the \emph{rejected} ones, and $\Phi^{\circ}_{n}$ the \emph{dormant} ones: those available to the architecture but never yet considered. We require $\Phi^{+}_{n}$ and $\Phi^{-}_{n}$ to be finite; $\Phi^{\circ}_{n}$ carries whatever unexplored extent the purpose space has.

To organize raw sensation into a coherent world model, an agent must actively select what is relevant and reject what is irrelevant. The data available to it at stage $n$ is accordingly
\begin{equation}\label{eq:data}
  D_{n} \;:=\; \Bigl\{\, s \in S \;\Bigm|\;
    \bigl(\exists \phi^+ \,(\phi^+ \in \Phi^{+}_{n} \wedge M(s,\phi^+))\bigr)
    \;\wedge\;
    \bigl(\forall \phi^- \,(\phi^- \in \Phi^{-}_{n} \rightarrow \neg M(s,\phi^-))\bigr)
    \,\Bigr\} ,
\end{equation}
which, unfolding \eqref{eq:dphi}, has the equivalent set-algebraic form
\begin{equation}\label{eq:data-alt}
  D_{n} \;=\; \Bigl(\bigcup_{\phi^+ \in \Phi^{+}_{n}} D_{\phi^+}\Bigr)
      \;\setminus\;
      \Bigl(\bigcup_{\phi^- \in \Phi^{-}_{n}} D_{\phi^-}\Bigr).
\end{equation}
A signal is admitted when at least one pursued purpose finds it meaningful and \emph{no} rejected purpose does. Dormant purposes impose no constraint: the agent excludes only what it has actively judged irrelevant, not what it has never judged at all. This is what keeps \eqref{eq:data} non-trivial as the repertoire grows. Were $\Phi^{-}_{n}$ to absorb everything not currently pursued, then for a sufficiently rich purpose space, almost every signal would be meaningful under some rejected purpose, and $D_{n}$ would collapse toward $\emptyset$---the richer the agent's repertoire, the less data would survive. The three-way partition avoids this: i.e., because $\Phi^{-}_{n}$ is finite, the universal clause in \eqref{eq:data} ranges over finitely many purposes at every stage.

We assume this construction is non-degenerate, $D_{n} \neq \emptyset$. In other words, the agent finds at least one signal meaningful. Note that this is an assumption on the architecture, not a consequence of \eqref{eq:data}, which is equally well satisfied by the empty set.

Equations \eqref{eq:data} and \eqref{eq:data-alt} make meaning exclusive by construction. A signal that is universally meaningful, $M(s,\phi)$ for all $\phi \in \Phi$, is excluded from $D_{n}$, since it is meaningful under every rejected purpose as well; relevance that does not discriminate between purposes
carries no information about purpose. This exclusivity has a direct algorithmic counterpart that will return in \S\ref{sec:what_the_result_does_and_does_not_establish}: just as a signal meaningful under every purpose carries no information about purpose, a program that enumerates every computable continuation without exclusion \citep{schmidhuber2000ate} makes no semantic discrimination.

The stage index is not decorative. As the agent develops, purposes move between the three blocks: a dormant purpose may be adopted, a pursued one abandoned. Each such movement redraws \eqref{eq:data-alt} and therefore changes which signals count as data. In particular, a signal excluded at stage $n$ because some rejected purpose found it meaningful may be admitted at stage $n+1$ once that purpose is no longer rejected. The mechanism governing these movements is developed in \S\ref{sec:recursive_teleology}; here, we note only that $D_{n}$ is a function of the agent's current teleological configuration, not a fixed feature of $S$.

\subsection{Why Fixed Purposes Eventually Saturate}\label{sec:saturation}

Machine learning architectures operate according to the classical definition given by Tom Mitchell \citep[p.~2]{mitchell1997machine}:

\begin{quote}
  \textit{A program learns from experience $E$ with respect to a task class $T$   and performance measure $P$ if its performance at tasks in $T$, measured by $P$, improves with $E$.}
\end{quote}

Fix a stage $n$ and a pursued purpose $\phi \in \Phi^{+}_{n}$ with task class
$T_{\phi}$. Experience is restricted to what that purpose deems meaningful and
what the agent's current discrimination admits,
\begin{equation}\label{eq:expdom}
  E \;\subseteq\; D_{\phi} \cap D_{n} \;=:\; D_{\phi,n} .
\end{equation}
The intersection here is not redundant: by \eqref{eq:data-alt} a signal meaningful under $\phi$ is nonetheless excluded from $D_{n}$ whenever some rejected purpose also finds it meaningful, so $D_{\phi} \not\subseteq D_{n}$ in general. Learning is therefore an optimization over a fixed semantic domain. Write
\begin{equation}\label{eq:perf}
  P_{\phi} \colon 2^{D_{\phi,n}} \to \mathbb{R},
\end{equation}
where $2^{X}$ denotes the power set of $X$ and $P_{\phi}(E)$ is the best performance on $T_{\phi}$ attainable from experience $E$. Two assumptions follow from Mitchell's definition:
\begin{align}
  \text{(A1)}\quad & E \subseteq E' \subseteq D_{\phi,n}
    \;\Longrightarrow\; P_{\phi}(E) \leq P_{\phi}(E'), \label{eq:a1}\\
  \text{(A2)}\quad & P_{\phi} \text{ is bounded above on } 2^{D_{\phi,n}}.
    \label{eq:a2}
\end{align}
(A1) states that more experience never hurts an idealized learner; (A2) holds trivially whenever $D_{\phi,n}$ is finite. Together, they yield a finite ceiling
\begin{equation}\label{eq:ceiling}
  P^{\ast}_{\phi} \;:=\; \sup\{\, P_{\phi}(E) : E \subseteq D_{\phi,n} \,\}
  \;<\; \infty ,
\end{equation}
attained at $E = D_{\phi,n}$ when $D_{\phi,n}$ is finite.
\footnote{The ceiling depends on the stage $n$ as well as on the purpose $\phi$, since $D_{\phi,n} = D_{\phi} \cap D_{n}$ and $D_{n}$ varies with the partition \eqref{eq:partition}; strictly it should be written $P^{\ast}_{\phi,n}$. We suppress the second index because the argument of here holds the configuration fixed, and because \S\ref{sec:recursive_teleology} compares ceilings across successor purposes rather than across stages of a single one. Where the distinction matters---a purpose whose ceiling rises because a rejection elsewhere was withdrawn, rather than because the purpose itself changed---the full notation should be restored.}
For any increasing stream of experience $E_{1} \subseteq E_{2} \subseteq \cdots \subseteq D_{\phi,n}$, the sequence $\bigl(P_{\phi}(E_{k})\bigr)_{k}$ is non-decreasing and bounded, hence convergent: i.e., performance saturates.

The plateau is a structural consequence of optimizing within a fixed semantic domain, not of limited computational resources. It persists under unbounded capacity and unbounded compute, because the ceiling $P^{\ast}_{\phi}$ depends only on $D_{\phi,n}$. Once $D_{\phi,n}$ has been sufficiently explored, further experience can refine existing distinctions but cannot generate new ones.

Progress resumes only if the teleological configuration itself changes. By \eqref{eq:partition} and \eqref{eq:data-alt}, there are two such movements. Promoting a dormant purpose, $\phi^\circ \in \Phi^{\circ}_{n} \to \Phi^{+}_{n+1}$, enlarges the first union of \eqref{eq:data-alt} and admits signals no pursued purpose previously found meaningful. Withdrawing a rejection, $\phi^- \in \Phi^{-}_{n} \to \Phi^{\circ}_{n+1}$, shrinks the second union and readmits signals that were barred despite being meaningful under a pursued purpose. Either movement changes $D_{n}$, hence $D_{\phi,n}$, hence the ceiling \eqref{eq:ceiling}, thereby exposing regions of $S$ that \eqref{eq:data} had discarded. The mechanism generating and selecting such movements is the subject
of \S\ref{sec:recursive_teleology}.

\subsection{Recursive Teleology}\label{sec:recursive_teleology}

The saturation of performance within a fixed purpose does not imply the exhaustion of the underlying signal space. It indicates that the current purpose has reached the limits of the distinctions it can express. Escaping the plateau, therefore, requires constructing a new purpose rather than accumulating further experience.

\paragraph{Saturation as a trigger.}

Fix a tolerance $\varepsilon > 0$. A state $(\phi_{n}, E_{n})$ with
$E_{n} \subseteq D_{\phi_{n},n}$ is \emph{$\varepsilon$-saturated} when
\begin{equation}\label{eq:sat}
  P^{\ast}_{\phi_{n}} - P_{\phi_{n}}(E_{n}) \;<\; \varepsilon ,
\end{equation}
with $P^{\ast}_{\phi}$ as in \eqref{eq:ceiling}. By \eqref{eq:a1} and \eqref{eq:a2}, every increasing stream of experience reaches this condition in finitely many steps. Saturation is what makes a teleological transition \emph{admissible}.

\paragraph{Generation of candidate purposes.}

Let $\mathcal{G}$ denote the purpose-transformation rules available to the architecture. Applied to the current purpose together with the history $H_{n} = (\phi_{0},\dots,\phi_{n})$ of purposes already adopted, these generate a set of \emph{candidate successor purposes}
\begin{equation}\label{eq:egen}
  \mathcal{E}_{\Phi}(\phi_{n}, H_{n})
  \;=\; \{\, g(\phi_{n}, H_{n}) : g \in \mathcal{G} \,\} ,
\end{equation}
which enter the repertoire as dormant,
\begin{equation}\label{eq:phigrow}
  \Phi_{n+1} \;=\; \Phi_{n} \cup \mathcal{E}_{\Phi}(\phi_{n}, H_{n}),
  \qquad
  \Phi_{\infty} \;=\; \bigcup_{n \geq 0} \Phi_{n} .
\end{equation}
Each $\Phi_{n}$ may be finite while $\Phi_{\infty}$ is not. Purposes are therefore \emph{constructed} from their predecessors rather than selected from a catalog fixed in advance, and the succession is not confined to a finite hierarchy (see the condition \S\ref{sec:local_convergence_and_global_openendedness} requires of open-endedness). We leave $\mathcal{G}$ largely unspecified at this stage, as \S\ref{sec:reflective_ext_and_semantic_surplus} leaves unspecified the transformation rules generating $\mathcal{E}(\mathcal{F}_{n})$; \eqref{eq:egen}
is the teleological instance of that construction.

One constraint on the set of rules $\mathcal{G}$ is imposed by the condition expressed in equation \eqref{eq:succ}. A successor must be able to accept signals that were not permitted by any previously adopted purpose, which means that $\mathcal{G}$ cannot consist solely of rules that narrow semantic domains. If every rule $g \in \mathcal{G}$ met the condition $D_{g(\phi,H)} \subseteq D_{\phi}$---as is the case with conjunctive refinement, since $D_{\phi \wedge \psi} = D_{\phi} \cap D_{\psi}$---then every candidate would fall within $D_{H_{n}}$. In that case, no successors would be admissible, and the process of succession would cease at $n = 0$. Recursive teleology, therefore, cannot be driven by purpose refinement alone: i.e., at least one rule must be non-monotone with respect to semantic domains,
\begin{equation}\label{eq:gcond}
  \exists\, g \in \mathcal{G} : \quad
  D_{g(\phi, H)} \not\subseteq D_{\phi} .
\end{equation}
This condition is necessary but not sufficient, and it exposes a limit of the present model when purposes are represented solely as subsets of a fixed signal space. If $\mathcal{G}$ generates a Boolean algebra over a finite set of primitive purposes, the induced family $\{D_{\phi}\}$ is closed under union, intersection, and complement, and therefore contains at most $2^{|S|}$ distinct domains. The strictly increasing sequence $\bigl(D_{H_{n}}\bigr)_{n}$ would then increase only finitely often, and saturation would reappear at the level of the purpose space itself---the very structure \S\ref{sec:scaling_paradox} identifies as the defect of the scaling paradigm, displaced upward by one level rather than removed. Sustaining teleological succession indefinitely, therefore, requires embedding purposes within a generative linguistic substrate that is not closed under a finite Boolean algebra. \S\ref{sec:toward_sematic_selfextension} and \S\ref{sec:formalization} supply that substrate by coupling teleological selection to the reflective extension of formal inferential systems, for which Gödelian incompleteness prevents any corresponding closure.

The same closure constraint explains why value-driven architectures whose goals live in a pre-structured vector space eventually saturate, and why appealing to dynamics in the cognitive architecture $C$ does not by itself supply a generative mechanism for new purposes. In the $(C, U, V)$ framework of \citet{he2026cuv}, learning concludes each cycle by ``updating the $V$-selector and $U$-selector, which govern the agent's ability to prioritize values and select appropriate skills. These selectors are the linchpins of the framework, enabling the agent to adapt its behavior to changing circumstances and objectives'' \citep[p.~172]{he2026cuv}. In the terms of \eqref{eq:partition}, updating the parameters of $V_{\mathrm{select}}$ merely redistributes pre-existing coordinates of $V$ between the dormant block $\Phi^{\circ}_{n}$ and the pursued block $\Phi^{+}_{n}$. Where \citet{he2026cuv} address how $(C, U, V)$ might transition between distinct value regimes under continuous dynamics $\mathrm{d}(C,U,V)/\mathrm{d}t = -\nabla \mathcal{F}(C,U,V)$, they attribute regime shifts to the fact that value feedback is ``inherently fuzzy and variable. This fuzziness and variability introduce a certain degree of jumpiness in the system's dynamics, such that the evolution process must deal with complex functional landscapes and realize the transitions of values and capabilities. To mathematically characterize and analyze this jumpiness, we can model it as a saddle escaping problem'' \citep[p.~177]{he2026cuv}. Yet perturbing a trajectory off a saddle point along a negative-curvature direction of $\nabla^{2}\mathcal{F}$ remains motion across a fixed functional landscape $\mathcal{F}$ defined over a fixed parameter space $(C, U, V)$, which under their Kurdyka--\L{}ojasiewicz condition still converges to a local minimum where learning halts. Noise on a fixed landscape can jar a selector into activating a different subset of already-defined coordinates, just as Boolean recombination within \eqref{eq:gcond} permutes a finite repertoire, but it does not specify how $C$ could construct a new purpose $\phi_{n+1}$ whose semantic domain $D_{\phi_{n+1}}$ lies outside $D_{H_{n}}$. Genuine teleological open-endedness requires that the meta-level rule set $\mathcal{G}$ be paired with an explicit \textit{a priori} meta-telos---the directive to construct new purposes whenever $D_{H_{n}}$ saturates---drawing candidate distinctions from the unresolved semantic surplus of the current regime.

\paragraph{Selection.}

Write $D_{H_{n}} = \bigcup_{k \leq n} D_{\phi_{k}}$ for the semantic territory the agent's purposes have jointly covered. A \emph{teleological transformation operator}\footnote{Note that $\Phi^{<\omega} = \bigcup_{k \geq 0} \Phi^{k}$ denotes the finite sequences over $\Phi$, so that the first argument is the history $H_{n}$ and the second the experience accumulated under it. The harpoon $\rightharpoonup$ indicates a partial map: the operator is defined only on $\varepsilon$-saturated states. Its value is a non-empty set of admissible successors rather than a single purpose, so that the architecture determines which continuations are available without determining which is taken.}
\begin{equation}\label{eq:top}
  \mathcal{T} \colon \Phi^{<\omega} \times 2^{S}
  \;\rightharpoonup\; 2^{\Phi_{\infty}} \setminus \{\emptyset\},
\end{equation}
defined only on $\varepsilon$-saturated states, returns the successors the architecture deems admissible,
\begin{equation}\label{eq:succ}
  \mathcal{T}(H_{n}, E_{n})
  \;\subseteq\;
  \bigl\{\, \phi \in \mathcal{E}_{\Phi}(\phi_{n}, H_{n})
    \;\bigm|\; D_{\phi} \not\subseteq D_{H_{n}} \,\bigr\},
  \qquad
  \phi_{n+1} \in \mathcal{T}(H_{n}, E_{n}).
\end{equation}
The operator determines which successors are admissible; it does not determine which is taken. The novelty condition $D_{\phi} \not\subseteq D_{H_{n}}$ is the substantive one. Novelty of the purpose alone, $\phi_{n+1} \notin H_{n}$, would not suffice: distinct purposes may induce identical semantic domains, in which case the successor exposes nothing its predecessors had not already admitted. Requiring $D_{\phi_{n+1}} \not\subseteq D_{H_{n}}$ makes $\bigl(D_{H_{n}}\bigr)_{n}$ strictly increasing and is the teleological counterpart of the persistent-novelty condition stated algorithmically as Assumption \ref{as:novelty} in \S\ref{sec:independent_semantic_extension}.

In parallel with the generative richness condition \eqref{eq:genrich} of \S\ref{sec:from_understanding_to_speculative_intelligence}, we require that saturation admit more than one continuation,
\begin{equation}\label{eq:rich4}
  \bigl|\mathcal{T}(H_{n}, E_{n})\bigr| \;\geq\; 2 ,
\end{equation}
which is strictly stronger than requiring two candidates in \eqref{eq:egen},
since the novelty condition in \eqref{eq:succ} may eliminate candidates. This is
an architectural requirement, not a consequence of saturation: nothing in
\eqref{eq:sat} guarantees that a plateau admits any continuation at all, let
alone two. When \eqref{eq:rich4} holds the trajectory branches, and the space of
possible trajectories acquires the structure invoked in \S\ref{sec:trajectory_complexity}.

\paragraph{State update.}

Adopting $\phi_{n+1}$ promotes it from the dormant block and revalues accumulated
experience,
\begin{align}
  \Phi^{+}_{n+1} &= \Phi^{+}_{n} \cup \{\phi_{n+1}\},
    \label{eq:promote}\\
  \Phi^{\circ}_{n+1} &= \bigl(\Phi^{\circ}_{n} \cup
    \mathcal{E}_{\Phi}(\phi_{n},H_{n})\bigr) \setminus
    \bigl(\{\phi_{n+1}\} \cup \Phi^{-}_{n+1}\bigr),
    \label{eq:dormant}\\
  E_{n+1} &\supseteq E_{n} \cap D_{\phi_{n+1},\,n+1} ,
    \label{eq:carry}
\end{align}
where $\Phi^{-}_{n+1}$ collects those candidates the architecture explicitly rejects, the remainder stays dormant. Equation \eqref{eq:carry} states that experience survives the transition exactly insofar as it remains meaningful under the successor and admissible under the new configuration. Since \eqref{eq:promote}--\eqref{eq:dormant} redraw the partition, the admissible data $D_{n}$ of \eqref{eq:data} is itself stage-dependent, and signals discarded at stage $n$ may become data at stage $n+1$. This is the formal content of the claim closing \S\ref{sec:saturation}.

\paragraph{Scope of the operator.}

$\mathcal{T}$ governs purpose \emph{succession} only. As noted in \S\ref{sec:saturation}, a
ceiling may also rise without any change of purpose---when a rejection
elsewhere is withdrawn, $\phi^- \in \Phi^{-}_{n} \to \Phi^{\circ}_{n+1}$, enlarging
$D_{n}$ and hence $D_{\phi_{n},n}$. Such reconfigurations leave $\phi_{n}$ in
place and fall outside \eqref{eq:top}. We treat them as a distinct class of
teleological transition and do not formalize them here; \S\ref{sec:reflective_ext_and_semantic_surplus} subsumes both under
the single operator $\mathcal{M}_{n}$ acting on the space of admissible
successors.

\paragraph{The two processes.}

Learning therefore alternates between accumulation within a purpose,
\begin{equation}\label{eq:within}
  (\phi_{n}, E_{n}) \rightsquigarrow (\phi_{n}, E_{n}'),
  \qquad E_{n} \subseteq E_{n}' \subseteq D_{\phi_{n},n},
\end{equation}
which is non-decreasing in $P_{\phi_{n}}$ by \eqref{eq:a1} and saturates by
\eqref{eq:sat}, and transition between purposes via
\eqref{eq:succ}--\eqref{eq:carry}, giving
\begin{equation}\label{eq:traj4}
  (\phi_{0}, E_{0}) \rightsquigarrow (\phi_{1}, E_{1})
  \rightsquigarrow (\phi_{2}, E_{2}) \rightsquigarrow \cdots
\end{equation}
Unlike classical machine learning, where experience accumulates under a fixed
task, recursive teleology allows the task itself to evolve. Intelligence is
therefore no longer described solely as the optimization of performance within a
semantic space, but as the recursive construction of increasingly expressive
semantic spaces through the generation of successor purposes.

\subsection{Toward Semantic Self-Extension}\label{sec:toward_sematic_selfextension}

The preceding sections establish a design rationale rather than a formal theory. Their content can be stated in one sentence:

\begin{quote}
An agent that discriminates must discard what it discards, and what it can learn is bound by what it discards, and escaping that bound requires transforming the principle of discrimination itself rather than accumulating further experience under it.
\end{quote}

Every act of cognition performs two complementary operations simultaneously. It commits to one interpretation, and in doing so, excludes countless others. Classical machine learning optimizes the committed interpretation; \eqref{eq:a1} and \eqref{eq:ceiling} describe how far that optimization can go. Here, we ask a different question: \textit{whether the excluded alternatives can themselves become the driving force of subsequent cognition}. Equation \eqref{eq:succ} gives a first answer, in which a successor purpose is admissible precisely when it recovers signals that every previous purpose had set aside.

That answer is, however, bounded whenever purposes are represented merely as bare subsets of $S$. As \eqref{eq:gcond} and the discussion following it make clear, the territory $D_{H_{n}}$ recoverable over a fixed vocabulary of primitive predicates is constrained by the Boolean closure of the generative repertoire $\mathcal{G}$, just as reweighting a fixed value vector $V$ is constrained by its pre-indexed dimensions, so that under any finite primitive basis the succession eventually terminates. Conversely, pure formal extension without teleological discrimination (e.g., a universal dovetailer enumerating all computable continuations with $\Phi^{-}_{n} = \emptyset$ \citep{schmidhuber2000ate}), possesses no principle for excluding branches and therefore collapses into unselective enumeration. Teleological succession and formal language extension consequently require each other: teleology without language extension saturates in a finite Boolean algebra, while language extension without teleology cannot select a trajectory.

We therefore deepen the register rather than abandoning teleology. In \S\ref{sec:formalization}, each stage is equipped with an explicit formal inferential system $\mathcal{F}_{n} = (L_{n}, \Gamma_{n}, \vdash_{n})$ and the space of admissible extensions $\mathcal{E}(\mathcal{F}_{n})$ constructible from it. The structural coupling between the two levels is exact: an active teleological configuration $\Phi_{n}$ is articulated within the inferential regime $\mathcal{F}_{n}$; the excluded signal region $S \setminus D_{n}$ has its formal counterpart in the unresolved region $I(O_{n})$; the candidate set $\mathcal{E}_{\Phi}(\phi_{n},H_{n})$ of \eqref{eq:egen} is generated through the admissible extensions $\mathcal{E}(\mathcal{F}_{n})$; and the teleological transformation operator $\mathcal{T}$ of \eqref{eq:top} is realized at the meta-level by the selection operator $\mathcal{M}_{n}$. Moving from bare subsets of $S$ to formal systems $\mathcal{F}_{n}$ provides the generative linguistic substrate through which teleological succession escapes Boolean saturation: when a reflective extension $\mathcal{R}_{n} = (\mathcal{F}_{n+1}, \tau_{n})$ expands the language from $L_{n}$ to $L_{n+1}$ in order to settle propositions in $I^{\mathrm{inc}}(O_{n})$, the newly introduced vocabulary $L_{n+1} \setminus \tau_{n}(L_{n})$ makes it possible to formulate relevance predicates $M(s, \phi_{n+1})$---such as the capacity to ``recognize the value of starting a fire by friction'' \citep[p.~178]{he2026cuv} discussed in \S\ref{sec:laws}---that could not be expressed as Boolean combinations of the primitive purposes available in $L_{n}$.

Two structural properties are gained by this coupling. First, the Boolean bound on extension disappears, because the admissible successors of a sufficiently expressive formal system are never closed under a finite set of primitive predicates; by \S\ref{sec:successor_incompleteness}, each successor $\mathcal{F}_{n+1}$ generates a fresh undecidable frontier $I^{\mathrm{inc}}(O_{n+1}) \neq \emptyset$. Second, the empirical trigger of $\varepsilon$-saturation in \eqref{eq:sat} and the logical guarantee of incompleteness in \eqref{eq:godel} are unified: as local contractive reasoning exhausts the decidable region $I^{\mathrm{loc}}(O_{n})$ and drives performance toward the saturation ceiling $P^{\ast}_{\phi_{n}}$, the non-empty undecidable frontier $I^{\mathrm{inc}}(O_{n})$ guaranteed by \eqref{eq:irreduce} remains intact, providing the permanent structural locus upon which an \textit{a priori} meta-teleology can act.

What follows in \S\ref{sec:formalization} is therefore the formal architecture that makes recursive teleology sustainable, grounding the generation of candidate frameworks in logical incompleteness while preserving meta-level selection $\mathcal{M}_{n}$ as the locus of teleological commitment. That coupling is carried forward directly into the collective and autonomous regimes of \S\ref{sec:collective_teleology} and \S\ref{sec:proxy_teleology_and_autonomous_intelligence}, where multiple agents co-recursively transform both their inferential languages and the purposes they pursue.

\section{Formalizing Recursive AGI Systems: An Axiomatic Framework}\label{sec:formalization}

This section introduces a formal model for a class of recursively self-modifying cognitive systems. The objective is to define an architecture that explicitly incorporates logical incompleteness and meta-level selection as foundational design principles, keeping local, contractive reasoning structurally distinct from global, open-ended semantic growth.

\subsection{Formal System, Cognitive State, and Incompleteness}\label{sec:formal_system_cognitive_state_and_incompleteness}

Let
\begin{equation}\label{eq:fsys}
  \mathcal{F}_{n} = (L_{n}, \Gamma_{n}, \vdash_{n})
\end{equation}
denote the inferential system available at recursive level $n$, where $L_{n}$ is a formal language, $\Gamma_{n} \subseteq L_{n}$ its axioms, and $\vdash_{n}$ the associated proof relation. We assume throughout that each $\mathcal{F}_{n}$ is \emph{effectively axiomatized} ($\Gamma_{n}$ is recursively enumerable), \emph{consistent}, and \emph{sufficiently expressive} in the sense that it interprets Robinson arithmetic and can therefore represent every primitive recursive function. These are the standard hypotheses of the Gödel--Rosser theorem.\footnote{\S\ref{sec:introduction} treats incompleteness as an engineering heuristic rather than as a theorem that the architecture must satisfy. From this section onward, we take the hypotheses literally, since the results of \S\ref{sec:successor_incompleteness} and \S\ref{sec:independent_semantic_extension} depend on them. The heuristic reading remains available for architectures that do not meet these conditions; what such architectures lose is the guarantee in \eqref{eq:irreduce}, not the recursive structure built upon it. We invoke Rosser's strengthening rather than Gödel's original formulation so that consistency alone suffices, without the additional hypothesis of $\omega$-consistency.}

A cognitive state at level $n$ is a \emph{partial} semantic evaluation
\begin{equation}\label{eq:cogstate}
  O_{n} \colon L_{n} \rightharpoonup \mathcal{V}_{n},
\end{equation}
where $\mathcal{V}_{n}$ is the set of semantic commitments currently available to the system. We do not require $\mathcal{V}_{n}$ to be two-valued: the polycontextural organization described in \S\ref{sec:laws} admits states in which several partially incompatible commitments coexist. The \emph{unresolved region} of the state is the complement of its domain,
\begin{equation}\label{eq:unresolved}
  I(O_{n}) \;=\; L_{n} \setminus \operatorname{dom}(O_{n})
  \;=\; \{\, x \in L_{n} \mid O_{n}(x) \text{ is undefined} \,\}.
\end{equation}
That $O_{n}$ is undefined at $x$ does not mean that $x$ is meaningless or false. It records a limitation of the current cognitive state: no determinate commitment regarding $x$ has been established.

The evaluation and the proof relation are connected by a single assumption.

\begin{assumption}[Inferential Soundness]\label{as:sound}
For every $x \in L_{n}$,
\begin{equation}\label{eq:sound}
  O_{n}(x) \text{ is defined}
  \quad\Longrightarrow\quad
  \mathcal{F}_{n} \vdash x \ \text{ or } \ \mathcal{F}_{n} \vdash \neg x .
\end{equation}
\end{assumption}

\noindent
The system does not commit on grounds unavailable to $\mathcal{F}_{n}$. The converse is deliberately not assumed. A state may leave a decidable proposition unresolved, since the system need not have drawn every consequence of its axioms; assuming otherwise would impose logical omniscience and collapse each level to a single state, leaving no local dynamics for \S\ref{sec:temporal_recursive_dynamics} and \S\ref{sec:local_stability_and_global_recursive_extension} to describe.

The unresolved region therefore decomposes into two parts of quite different
character. Write
\begin{equation}\label{eq:decomp}
  I(O_{n}) \;=\; I^{\mathrm{loc}}(O_{n}) \;\dot\cup\; I^{\mathrm{inc}}(O_{n}),
\end{equation}
where
\begin{align}
  I^{\mathrm{loc}}(O_{n}) &= \{\, x \in I(O_{n}) \mid
    \mathcal{F}_{n} \vdash x \text{ or } \mathcal{F}_{n} \vdash \neg x \,\},
    \label{eq:iloc}\\
  I^{\mathrm{inc}}(O_{n}) &= \{\, x \in I(O_{n}) \mid
    \mathcal{F}_{n} \nvdash x \text{ and } \mathcal{F}_{n} \nvdash \neg x \,\}.
    \label{eq:iinc}
\end{align}
The first part is unresolved because the system has not yet done the work; the second is because the work cannot be done within $\mathcal{F}_{n}$ at all. Local transformation within a regime (\S\ref{sec:temporal_recursive_dynamics}) reduces $I^{\mathrm{loc}}$ and leaves $I^{\mathrm{inc}}$ untouched.

That $I^{\mathrm{inc}}$ is inhabited reflects the essence of incompleteness. According to the Gödel--Rosser theorem, any $\mathcal{F}_{n}$ that meets the above hypotheses allows for a sentence $G_{n} \in L_{n}$ with
\begin{equation}\label{eq:godel}
  \mathcal{F}_{n} \nvdash G_{n}
  \qquad\text{and}\qquad
  \mathcal{F}_{n} \nvdash \neg G_{n},
\end{equation}
whence, by Assumption~\ref{as:sound} and \eqref{eq:iinc},
\begin{equation}\label{eq:irreduce}
  G_{n} \in I^{\mathrm{inc}}(O_{n}),
  \qquad\text{so}\qquad
  I^{\mathrm{inc}}(O_{n}) \neq \emptyset .
\end{equation}

This is the first recursive pressure of the theory. An unresolved region, when non-empty, is not particularly remarkable on its own; any system that has not completed its reasoning will have one. However, what \eqref{eq:irreduce} establishes is stronger: i.e., part of that region is immune to further reasoning within the current system, regardless of the extent of reasoning performed. A system capable of representing its own reasoning thus faces a boundary that cannot be eliminated through continued inference at the same level. Instead, addressing this boundary requires a change in the inferential regime.

Two qualifications limit what \eqref{eq:irreduce} delivers on its own. Taken purely extensionally, it asserts only that $I^{\mathrm{inc}}(O_{n})$ is non-empty; it does not by itself guarantee that this unresolved region is large, structured, or productive, nor can a non-omniscient system without an external oracle inspect an arbitrary $x \in I(O_{n})$ in advance to tell whether $x \in I^{\mathrm{loc}}(O_{n})$ or $x \in I^{\mathrm{inc}}(O_{n})$ prior to attempting a derivation. The condition stating that unresolved structure allows for at least two admissible continuations is stated for general architectures in \eqref{eq:genrich} and does not follow from bare incompleteness alone; \S\ref{sec:reflective_ext_and_semantic_surplus} shows how, in any architecture that reflectively reifies its own failed proof searches, a multi-branched semantic surplus is constructively generated from the operational trace of attempting to resolve $I^{\mathrm{inc}}(O_{n})$. Additionally, $G_{n}$ serves as a witness rather than a description: $I^{\mathrm{inc}}(O_{n})$ will generally include numerous sentences of much greater cognitive interest than a self-referential construction. It is these more interesting sentences that motivate the extensions discussed in \S\ref{sec:reflective_ext_and_semantic_surplus}.

\subsection{Reflective Extension and Semantic Surplus}\label{sec:reflective_ext_and_semantic_surplus}

Resolving a proposition in $I^{\mathrm{inc}}(O_{n})$ requires, by \eqref{eq:iinc}, an extension of the inferential resources available to the system. No amount of local transformation reaches it. We therefore consider the construction of successor systems.

\paragraph{Admissible extensions.}

Let $\mathcal{E}(\mathcal{F}_{n})$ denote the set of \emph{admissible successor systems} constructible from $\mathcal{F}_{n}$ under the transformation rules available to the architecture. A successor
\begin{equation}\label{eq:admissible}
  \mathcal{F}_{n+1} = (L_{n+1}, \Gamma_{n+1}, \vdash_{n+1})
  \;\in\; \mathcal{E}(\mathcal{F}_{n})
\end{equation}
is admissible when it satisfies two conditions. First, \emph{regularity}: $\mathcal{F}_{n+1}$ is itself effectively axiomatized, consistent, and
sufficiently expressive in the sense of \S\ref{sec:formal_system_cognitive_state_and_incompleteness}. Second, \emph{reflective adequacy}: the construction supplies a translation
\begin{equation}\label{eq:reflect}
  \tau_{n} \colon L_{n} \longrightarrow L_{n+1}
\end{equation}
under which the successor can represent the syntax and proof relation of its predecessor, so that statements about $\mathcal{F}_{n}$ are expressible in $L_{n+1}$. We write $\mathcal{R}_{n} = (\mathcal{F}_{n+1}, \tau_{n})$ for the resulting \emph{recursive extension}, carrying both the successor system and the translation into it; \S\ref{sec:recursive_extension} develops its properties.

Reflective adequacy is less stringent than inclusion. We do not require that
\begin{equation}\label{eq:noinclusion}
\mathcal{F}_{n} \subseteq \mathcal{F}_{n+1},
\end{equation}
nor must $\tau_{n}$ preserve the interpretation of what it translates. A successor may clarify an ambiguity present in its predecessor, maintain it, reinterpret it, or introduce new ambiguities of its own. However, it must not lose the ability to discuss its origins: an architecture that cannot represent its own history cannot reflect on it. \S\ref{sec:recursive_extension} and \S\ref{sec:collective_memory_extended_cognitive_space} both rely on this principle. As a result, recursive intelligence is not defined by guaranteed convergence toward a single, complete representation but rather by the capacity to construct and reason about successive representational systems.

\paragraph{Semantic surplus.}

The set $\mathcal{E}(\mathcal{F}_{n})$ is generally not a singleton. Where it is not, the architecture faces several admissible ways of continuing, none of which $\mathcal{F}_{n}$ itself determines. We call this underdetermination the \emph{semantic surplus} of the state and write
\begin{equation}\label{eq:surplus}
  \mathrm{Surp}(O_{n}) \;=\; \mathcal{E}(\mathcal{F}_{n}),
  \qquad
  \text{with surplus \emph{present} when }
  \bigl|\mathrm{Surp}(O_{n})\bigr| \geq 2 .
\end{equation}
Semantic surplus is therefore neither randomness nor an arbitrarily large set of alternative answers. It is the space of admissible transformations left open relative to the current formal system. Its cardinality measures how many distinct continuations the architecture has available, which is the quantity that the collective analysis of \S\ref{sec:collective_generative_capacity} aggregates across agents.

Within this space, we distinguish the successors that make a difference to what is decidable. Let
\begin{equation}\label{eq:decisive}
  \mathrm{Surp}^{\mathrm{dec}}(O_{n}) \;=\;
  \bigl\{\, \mathcal{F}' \in \mathrm{Surp}(O_{n}) \;\bigm|\;
  \exists\, x \in I^{\mathrm{inc}}(O_{n}) :
  \mathcal{F}' \vdash \tau_{n}(x) \ \text{ or } \
  \mathcal{F}' \vdash \neg\tau_{n}(x) \,\bigr\}
\end{equation}
denote the \emph{decisive} surplus: those admissible successors that settle some proposition their predecessor could not. The inclusion $\mathrm{Surp}^{\mathrm{dec}}(O_{n}) \subseteq \mathrm{Surp}(O_{n})$ may be strict, since a successor may reorganize a representation without deciding anything new. Both notions are needed. Generative capacity depends on the cardinality of the full surplus, whereas the recursive pressure of \S\ref{sec:formal_system_cognitive_state_and_incompleteness} is relieved only by its decisive part.

This gives the third law of \S\ref{sec:laws} a formal referent. Surplus is a resource in a precise sense: it is what the meta-level operates \emph{on}. Where $|\mathrm{Surp}(O_{n})| \leq 1$, there is nothing to evaluate, compare, or select, and the architecture has no move available beyond local transformation---which by \eqref{eq:irreduce} cannot reach $I^{\mathrm{inc}}$.

\paragraph{Proof-attempt traces and the conversion of form into content.}

Why should encountering an undecidable frontier generate a plurality of admissible extensions in $\mathrm{Surp}^{\mathrm{dec}}(O_{n})$ rather than a mere halting boundary? And why does the derivation attempt itself perform non-trivial mathematical work beyond simply adjoining $x$ or $\neg x$ by fiat? Considered purely extensionally, \eqref{eq:irreduce} establishes only the negative fact that some $x \in I^{\mathrm{inc}}(O_{n})$ cannot be proved or refuted in $\mathcal{F}_{n}$, and Lindenbaum's lemma already implies that adjoining $x$ or $\neg x$ directly to $\Gamma_{n}$ yields two consistent extensions. Yet such a direct adjunction treats $x$ as an unanalyzed atom. In an actual cognitive system, by contrast, two structural features make the proof search indispensable. First, because theoremhood in $\mathcal{F}_{n}$ is only recursively enumerable, the system is not handed the partition $I(O_{n}) = I^{\mathrm{loc}}(O_{n}) \dot\cup I^{\mathrm{inc}}(O_{n})$ by an oracle and cannot know prior to attempting a derivation whether an unresolved sentence $x \in I(O_{n})$ lies in $I^{\mathrm{loc}}(O_{n})$ or in $I^{\mathrm{inc}}(O_{n})$. Second, attempting to prove or refute $x$ decomposes the surface question into intermediate subgoals and conditional dependencies, exposing the underlying lemmas on which $x$ actually turns.

To separate what holds mathematically of the unresolved frontier from what a finite proof search can operationally compute, we formalize the structure of a non-trivial proof attempt. For any unresolved sentence $x \in I(O_{n})$ and finite search budget $t \geq 1$, let the \emph{derivation trace} generated by the object-level proof searcher up to step $t$ be
\begin{equation}\label{eq:prooftrace}
  \pi_{n,t}(x) \;=\; \bigl(\mathcal{D}_{n,t}(x),\; \mathcal{B}_{n,t}(x)\bigr),
\end{equation}
where $\mathcal{D}_{n,t}(x)$ is the finite directed acyclic graph of inference steps explored between $\Gamma_{n}$ and $\{x, \neg x\}$ within $t$ steps, and
\begin{equation}\label{eq:blocking}
  \mathcal{B}_{n,t}(x) \;=\; \{b_{1}, \dots, b_{m}\} \;\subseteq\; I(O_{n}) \setminus \{x, \neg x\}
\end{equation}
is the finite \emph{blocking frontier} of proper intermediate subgoals (excluding the trivial self-reduction $b = x$ or $b = \neg x$) at which $\mathcal{D}_{n,t}(x)$ stalls at step $t$. We say that $\pi_{n,t}(x)$ is a \emph{proper reduction trace} if the verified inference links in $\mathcal{D}_{n,t}(x)$ establish in $\mathcal{F}_{n}$ a truth-functional reduction of $x$ to $\mathcal{B}_{n,t}(x)$, a Boolean formula $\Psi_{\pi}(u_{1}, \dots, u_{m})$ synthesized from the conditional branches of $\mathcal{D}_{n,t}(x)$ such that
\begin{equation}\label{eq:reduction}
  \Gamma_{n} \;\vdash_{n}\; \bigl(x \;\leftrightarrow\; \Psi_{\pi}(b_{1}, \dots, b_{m})\bigr).
\end{equation}
Condition \eqref{eq:reduction} is satisfied by any standard goal-reduction calculus (such as backward sequent decomposition, resolution refutation, or horn-clause unfolding) after $t \geq 1$ non-trivial expansion steps: the interior nodes of $\mathcal{D}_{n,t}(x)$ record theorems of $\mathcal{F}_{n}$ linking $x$ to the frontier leaves $b_{1}, \dots, b_{m} \in I(O_{n}) \setminus \{x, \neg x\}$. When $x \in I^{\mathrm{loc}}(O_{n})$, a sufficiently large budget $t$ discharges enough leaves $b_{i}$ in $\Gamma_{n}$ to evaluate $\Psi_{\pi}$ to a theorem or refutation, closing the trace ($\mathcal{B}_{n,t}(x) = \emptyset$). When $x \in I^{\mathrm{inc}}(O_{n})$, no finite $t$ can close the trace.

Within the object regime $\mathcal{F}_{n}$, an unclosed trace $\pi_{n,t}(x)$ yields no unconditional theorem and leaves $O_{n}(x)$ undefined. Hegel's analysis of dialectical method in the \emph{Phenomenology of Spirit} and \emph{The Science of Logic} identifies why treating such an operational failure as a mere nullity mistakes the nature of cognition. A cognitive process cannot be reduced to its bare terminal theorem in isolation from the derivation attempted to reach it, ``{[}f{]}or the real issue is not exhausted by stating it as an aim, but by carrying it out, nor is the result the actual whole, but rather the result together with the process through which it came about'' \citep[p.~2]{hegel1977phenomenology}. When a mode of knowing stalls at its internal limit, its failure is not an empty nothingness, ``{[}f{]}or it is only when it is taken as the result of that from which it emerges, that it is, in fact, the true result; in that case it is itself a determinate nothingness, one which has a content'' \citep[p.~51]{hegel1977phenomenology}. In \emph{The Science of Logic}, Hegel formulates this principle as the general mechanism of conceptual progression: ``what is self-contradictory does not resolve itself into a nullity, into abstract nothingness, but essentially only into the negation of its particular content; or that such a negation is not just negation, but is the negation of the determined fact which is resolved, and is therefore determinate negation; that in the result there is therefore contained in essence that from which the result derives'' \citep[p.~33]{hegel2010}. Because ``method is the consciousness of the form of the inner self-movement of the content of logic'' \citep[p.~33]{hegel2010} and ``content is at the same time the negative reflection of the form determinations into themselves'' \citep[p.~396]{hegel2010}, reflecting on the failed operation converts its form into a new object of thought: ``the new object shows itself to have come about through a reversal of consciousness itself'' \citep[p.~55]{hegel1977phenomenology}, and through that reflection ``the subject matter has received a determinateness for itself --- and this determinateness is a content, for the negativity that has withdrawn into simplicity is the sublated form'' \citep[p.~748]{hegel2010}.

Reflective adequacy \eqref{eq:reflect} gives this Hegelian conversion of operational form into semantic content a direct formal realization. Under $\tau_{n} \colon L_{n} \to L_{n+1}$, the unclosed derivation trace $\pi_{n,t}(x)$---which at level $n$ was only the operational \emph{form} of an incomplete proof search---is reified in $L_{n+1}$ as a syntactic object $\ulcorner \pi_{n,t}(x) \urcorner$ accompanied by a finite reflective diagram $\Phi_{\mathrm{tr}}\bigl(\pi_{n,t}(x)\bigr) \subseteq L_{n+1}$ encoding the verified inference links of $\mathcal{D}_{n,t}(x)$ and the reduction equivalence $\tau_{n}(x) \leftrightarrow \Psi_{\pi}\bigl(\tau_{n}(b_{1}), \dots, \tau_{n}(b_{m})\bigr)$. What was a failed operational process in $\mathcal{F}_{n}$ thereby becomes determinate propositional \emph{content} in $L_{n+1}$. For any blocking premise $b \in \mathcal{B}_{n,t}(x)$ isolated by the trace, reification induces a canonical pair of candidate trace-extensions,
\begin{equation}\label{eq:formcontent}
\begin{aligned}
  \widehat{\mathcal{F}}^{+}_{n+1}(b) \;&=\; \bigl(L_{n+1},\; \tau_{n}(\Gamma_{n}) \cup \Phi_{\mathrm{tr}}\bigl(\pi_{n,t}(x)\bigr) \cup \{\tau_{n}(b)\},\; \vdash_{n+1}\bigr), \\[3pt]
  \widehat{\mathcal{F}}^{-}_{n+1}(b) \;&=\; \bigl(L_{n+1},\; \tau_{n}(\Gamma_{n}) \cup \Phi_{\mathrm{tr}}\bigl(\pi_{n,t}(x)\bigr) \cup \{\neg\tau_{n}(b)\},\; \vdash_{n+1}\bigr),
\end{aligned}
\end{equation}
which preserve the predecessor's validated inferential structure while resolving the specific intermediate blockage $b \neq x$ in opposite directions.

We now establish two distinct results corresponding to the logical and operational levels of this construction: first, an existence theorem showing that any proper reduction trace of an undecidable sentence $x \in I^{\mathrm{inc}}(O_{n})$ must expose at least one genuinely undecidable intermediate premise $b \neq x$ that generates non-trivial branching in $\mathrm{Surp}^{\mathrm{dec}}(O_{n})$; and second, an operational proposition characterizing what happens when a non-omniscient system applies \eqref{eq:formcontent} to finite open traces without knowing in advance which leaves belong to $I^{\mathrm{inc}}(O_{n})$.

\begin{proposition}[Existence of Non-Trivial Trace-Reduced Decisive Extensions]\label{prop:tracesurplus}
Let $\mathcal{F}_{n} = (L_{n}, \Gamma_{n}, \vdash_{n})$ be a regular formal system and let $O_{n}$ be a cognitive state satisfying Assumption~\ref{as:sound}. Suppose the admissible transformations $\mathcal{E}(\mathcal{F}_{n})$ include the consistent members of the trace-objectifying family \eqref{eq:formcontent} for proper reduction traces $\pi_{n,t}(x)$ satisfying \eqref{eq:blocking}--\eqref{eq:reduction} whose reflective diagram $\Phi_{\mathrm{tr}}\bigl(\pi_{n,t}(x)\bigr)$ is conservative over $\tau_{n}(\Gamma_{n})$. Then for every $x \in I^{\mathrm{inc}}(O_{n})$ and every proper reduction trace $\pi_{n,t}(x)$:
\begin{enumerate}
  \item \emph{(Non-trivial undecidable frontier.)} The blocking frontier $\mathcal{B}_{n,t}(x) \subseteq I(O_{n}) \setminus \{x, \neg x\}$ contains at least one non-trivial undecidable premise,
  \begin{equation}\label{eq:nonemptycut}
    \mathcal{B}_{n,t}(x) \cap I^{\mathrm{inc}}(O_{n}) \;\neq\; \emptyset .
  \end{equation}
  \item \emph{(Branching decisive surplus.)} For every $b \in \mathcal{B}_{n,t}(x) \cap I^{\mathrm{inc}}(O_{n})$, both extensions $\mathcal{F}^{+}_{n+1}(b) := \widehat{\mathcal{F}}^{+}_{n+1}(b)$ and $\mathcal{F}^{-}_{n+1}(b) := \widehat{\mathcal{F}}^{-}_{n+1}(b)$ in \eqref{eq:formcontent} are regular, reflectively adequate, and decisive over $b \neq x$, so that
  \begin{equation}\label{eq:tracesurplus}
    \bigl\{\mathcal{F}^{+}_{n+1}(b),\; \mathcal{F}^{-}_{n+1}(b)\bigr\} \;\subseteq\; \mathrm{Surp}^{\mathrm{dec}}(O_{n})
    \qquad\text{and}\qquad
    \bigl|\mathrm{Surp}(O_{n})\bigr| \;\geq\; \bigl|\mathrm{Surp}^{\mathrm{dec}}(O_{n})\bigr| \;\geq\; 2 .
  \end{equation}
  Whenever $k \geq 1$ blocking premises $b_{1}, \dots, b_{k} \in \mathcal{B}_{n,t}(x) \cap I^{\mathrm{inc}}(O_{n})$ are mutually independent over $\Gamma_{n}$, simultaneous resolution of $\mathcal{B}_{n,t}(x) \cap I^{\mathrm{inc}}(O_{n})$ generates $2^{k} \geq 2$ distinct decisive extensions in $\mathrm{Surp}^{\mathrm{dec}}(O_{n})$, each of which also decides $\tau_{n}(x)$ via \eqref{eq:reduction} once the remaining leaves in $\mathcal{B}_{n,t}(x) \cap I^{\mathrm{loc}}(O_{n})$ are resolved by local deduction.
\end{enumerate}
\end{proposition}

\begin{proof}
\begin{enumerate}
    \item[\textbf{(1)}] For the first part, suppose toward a contradiction that $\mathcal{B}_{n,t}(x) \cap I^{\mathrm{inc}}(O_{n}) = \emptyset$. 
    Since $\mathcal{B}_{n,t}(x) \subseteq I(O_{n}) = I^{\mathrm{loc}}(O_{n}) \;\dot\cup\; I^{\mathrm{inc}}(O_{n})$ by \eqref{eq:decomp}, 
    every $b_{i} \in \mathcal{B}_{n,t}(x)$ must lie in $I^{\mathrm{loc}}(O_{n})$. 
    By \eqref{eq:iloc}, each $b_{i}$ is either provable or refutable from $\Gamma_{n}$ (i.e., $\Gamma_{n} \vdash_{n} b_{i}$ or $\Gamma_{n} \vdash_{n} \neg b_{i}$). 
    Substituting these provable literals into the truth-functional equivalence 
    $\Gamma_{n} \vdash_{n} \bigl(x \leftrightarrow \Psi_{\pi}(b_{1}, \dots, b_{m})\bigr)$ from \eqref{eq:reduction} 
    fixes the truth value of $\Psi_{\pi}$ in $\mathcal{F}_{n}$, yielding either $\Gamma_{n} \vdash_{n} x$ or $\Gamma_{n} \vdash_{n} \neg x$, 
    which contradicts $x \in I^{\mathrm{inc}}(O_{n})$. 
    Thus, \eqref{eq:nonemptycut} holds. Furthermore, because $x, \neg x \notin \mathcal{B}_{n,t}(x)$ by \eqref{eq:blocking}, 
    every $b \in \mathcal{B}_{n,t}(x) \cap I^{\mathrm{inc}}(O_{n})$ is distinct from $x$ and $\neg x$.

    \item[\textbf{(2)}] For the second part, fix $b \in \mathcal{B}_{n,t}(x) \cap I^{\mathrm{inc}}(O_{n})$. 
    Because $\mathcal{F}_{n} \nvdash b$ and $\mathcal{F}_{n} \nvdash \neg b$, and since $\Phi_{\mathrm{tr}}\bigl(\pi_{n,t}(x)\bigr)$ 
    is conservative over $\tau_{n}(\Gamma_{n})$, neither $\tau_{n}(b)$ nor its negation is refutable from 
    $\tau_{n}(\Gamma_{n}) \cup \Phi_{\mathrm{tr}}\bigl(\pi_{n,t}(x)\bigr)$. 
    Since $\tau_{n}(\Gamma_{n})$ is recursively enumerable and $\Phi_{\mathrm{tr}}\bigl(\pi_{n,t}(x)\bigr) \cup \{\pm\tau_{n}(b)\}$ is finite, 
    both $\mathcal{F}^{+}_{n+1}(b)$ and $\mathcal{F}^{-}_{n+1}(b)$ are consistent, effectively axiomatized, and sufficiently 
    expressive via $\tau_{n}$ (hence regular and reflectively adequate members of $\mathcal{E}(\mathcal{F}_{n})$). 
    Because $\mathcal{F}^{+}_{n+1}(b) \vdash_{n+1} \tau_{n}(b)$ and $\mathcal{F}^{-}_{n+1}(b) \vdash_{n+1} \neg\tau_{n}(b)$ 
    with $b \in I^{\mathrm{inc}}(O_{n})$, both extensions belong to $\mathrm{Surp}^{\mathrm{dec}}(O_{n})$ by \eqref{eq:decisive}, 
    establishing \eqref{eq:tracesurplus}.

    \item[\textbf{(3)}] Finally, if $b_{1}, \dots, b_{k} \in \mathcal{B}_{n,t}(x) \cap I^{\mathrm{inc}}(O_{n})$ are mutually independent over $\Gamma_{n}$, 
    all $2^{k}$ sign assignments $\{\pm\tau_{n}(b_{1}), \dots, \pm\tau_{n}(b_{k})\}$ are consistent with 
    $\tau_{n}(\Gamma_{n}) \cup \Phi_{\mathrm{tr}}\bigl(\pi_{n,t}(x)\bigr)$ and define $2^{k}$ distinct decisive successors. 
    Since every remaining local element $b_{j} \in \mathcal{B}_{n,t}(x) \cap I^{\mathrm{loc}}(O_{n})$ is already decided by $\Gamma_{n}$, 
    the translated equivalence $\tau_{n}(x) \leftrightarrow \Psi_{\pi}\bigl(\tau_{n}(b_{1}), \dots, \tau_{n}(b_{m})\bigr)$ 
    in $\Phi_{\mathrm{tr}}\bigl(\pi_{n,t}(x)\bigr)$ uniquely decides $\tau_{n}(x)$ in each of the $2^{k}$ extensions.
\end{enumerate}
\end{proof}

Proposition~\ref{prop:tracesurplus} establishes the existence of non-trivial decisive extensions, but its statement quantifies over $b \in \mathcal{B}_{n,t}(x) \cap I^{\mathrm{inc}}(O_{n})$. Because $I^{\mathrm{inc}}(O_{n})$ is not recursively decidable from finite derivation attempts, an actual agent at step $t$ sees only the unclosed blocking frontier $\mathcal{B}_{n,t}(x) \subseteq I(O_{n}) \setminus \{x, \neg x\}$ as a whole, without knowing which leaves $b \in \mathcal{B}_{n,t}(x)$ are genuinely undecidable ($b \in I^{\mathrm{inc}}(O_{n})$) and which have merely not yet been proved or refuted within $t$ steps ($b \in I^{\mathrm{loc}}(O_{n})$). The next proposition characterizes the operational behavior of trace-guided candidate generation under this epistemic constraint.

\begin{proposition}[Trace-Guided Candidate Generation and Fallibility]\label{prop:tracefallible}
Let $\mathcal{F}_{n}$ be a regular formal system, let $O_{n}$ satisfy Assumption~\ref{as:sound}, and let $\pi_{n,t}(x) = \bigl(\mathcal{D}_{n,t}(x), \mathcal{B}_{n,t}(x)\bigr)$ be an unclosed proper reduction trace at finite search step $t \geq 1$. When the architecture applies trace reification \eqref{eq:formcontent} to an arbitrary open premise $b \in \mathcal{B}_{n,t}(x)$ whose status in $I^{\mathrm{loc}}(O_{n}) \dot\cup I^{\mathrm{inc}}(O_{n})$ is unknown at step $t$, the resulting candidate pair $\bigl(\widehat{\mathcal{F}}^{+}_{n+1}(b), \widehat{\mathcal{F}}^{-}_{n+1}(b)\bigr)$ satisfies the following exact trichotomy:
\begin{enumerate}
  \item \emph{(Latent provability: $b \in I^{\mathrm{loc}}(O_{n})$ with $\mathcal{F}_{n} \vdash b$.)} The negative candidate $\widehat{\mathcal{F}}^{-}_{n+1}(b)$ is inconsistent, whereas the positive candidate $\widehat{\mathcal{F}}^{+}_{n+1}(b)$ is a conservative deductive shortcut over $\tau_{n}(\Gamma_{n}) \cup \Phi_{\mathrm{tr}}\bigl(\pi_{n,t}(x)\bigr)$ that decides no sentence of $I^{\mathrm{inc}}(O_{n})$ via $b$.
  \item \emph{(Latent refutability: $b \in I^{\mathrm{loc}}(O_{n})$ with $\mathcal{F}_{n} \vdash \neg b$.)} Symmetrically, $\widehat{\mathcal{F}}^{+}_{n+1}(b)$ is inconsistent, whereas $\widehat{\mathcal{F}}^{-}_{n+1}(b)$ is a conservative deductive shortcut.
  \item \emph{(Genuine undecidability: $b \in I^{\mathrm{inc}}(O_{n})$.)} Both $\widehat{\mathcal{F}}^{+}_{n+1}(b)$ and $\widehat{\mathcal{F}}^{-}_{n+1}(b)$ are consistent, regular, and belong to $\mathrm{Surp}^{\mathrm{dec}}(O_{n})$ by Proposition~\ref{prop:tracesurplus}(2).
\end{enumerate}
Moreover, if $x \in I^{\mathrm{inc}}(O_{n})$, then generating the candidate family $\bigl\{\widehat{\mathcal{F}}^{\pm}_{n+1}(b) \mid b \in \mathcal{B}_{n,t}(x)\bigr\}$ across the finite frontier $\mathcal{B}_{n,t}(x)$ is guaranteed to include at least one genuine decisive pair from Case~(3), even though no effective procedure within $\mathcal{F}_{n}$ can decide in advance which $b \in \mathcal{B}_{n,t}(x)$ falls under Case~(3) or certify the global consistency of $\widehat{\mathcal{F}}^{\pm}_{n+1}(b)$.
\end{proposition}

\begin{proof}
Since $b \in \mathcal{B}_{n,t}(x) \subseteq I(O_{n})$, the decomposition \eqref{eq:decomp} partitions $b$ exhaustively into Cases~(1)--(3). 

In Case~(1), $\Gamma_{n} \vdash_{n} b$ implies $\tau_{n}(\Gamma_{n}) \vdash_{n+1} \tau_{n}(b)$ by reflective adequacy \eqref{eq:reflect}; adjoining $\neg\tau_{n}(b)$ in $\widehat{\mathcal{F}}^{-}_{n+1}(b)$ therefore yields an inconsistent theory, while adjoining $\tau_{n}(b)$ in $\widehat{\mathcal{F}}^{+}_{n+1}(b)$ adds a theorem already entailed by $\tau_{n}(\Gamma_{n})$. 

Case~(2) holds by replacing $b$ with $\neg b$. 

Case~(3) is Proposition~\ref{prop:tracesurplus}.2. When $x \in I^{\mathrm{inc}}(O_{n})$, Proposition~\ref{prop:tracesurplus}.1 ensures $\mathcal{B}_{n,t}(x) \cap I^{\mathrm{inc}}(O_{n}) \neq \emptyset$, so at least one $b \in \mathcal{B}_{n,t}(x)$ falls under Case~(3). 

Finally, deciding within $\mathcal{F}_{n}$ whether an arbitrary $b \in \mathcal{B}_{n,t}(x)$ belongs to $I^{\mathrm{loc}}(O_{n})$ or $I^{\mathrm{inc}}(O_{n})$ would decide theoremhood in $\mathcal{F}_{n}$, which is undecidable for any regular system by Church's theorem, and certifying the consistency of a consistent extension $\widehat{\mathcal{F}}^{\pm}_{n+1}(b)$ within $\mathcal{F}_{n}$ is excluded by Gödel's second incompleteness theorem.
\end{proof}

Propositions~\ref{prop:tracesurplus} and~\ref{prop:tracefallible} jointly clarify why the transition from incompleteness to semantic surplus is neither automatic in a non-reflective calculus nor reducible to an infallible algorithm. Proposition~\ref{prop:tracesurplus} guarantees that whenever a proof search reduces an undecidable question $x \in I^{\mathrm{inc}}(O_{n})$ to a non-trivial blocking frontier $\mathcal{B}_{n,t}(x)$, genuinely branching decisive extensions exist in $\mathrm{Surp}^{\mathrm{dec}}(O_{n})$ that act on the intermediate lemmas $b \neq x$ uncovered by the trace rather than merely stipulating $x$ or $\neg x$ by fiat. Proposition~\ref{prop:tracefallible} shows why turning those open traces into successor systems is an inherently speculative, fallible operation: because a finite search cannot distinguish $I^{\mathrm{inc}}(O_{n})$ from $I^{\mathrm{loc}}(O_{n})$ by inspection, trace-guided candidate generation inevitably produces a mixture of genuine decisive extensions (Case\textasciitilde3), conservative shortcuts, and latent inconsistencies (Cases\textasciitilde1--2) that can only be weeded out through subsequent inference, relative consistency checks, and empirical or teleological evaluation by $\mathcal{M}_{n}$. This is the exact mathematical reason why an open-ended architecture requires both polycontexturality (\S\ref{sec:laws}) and distributed collective exploration (\S\ref{sec:collective_corecursion}): competing trace-extensions $\widehat{\mathcal{F}}^{\pm}_{n+1}(b)$ must be pursued across parallel contextures or agents so that the collapse of an inconsistent branch in Cases\textasciitilde(1)--(2) refutes only that candidate contexture while preserving the viable decisive branches of Case\textasciitilde(3). As Hegel writes of determinate negation, ``{[}b{]}ecause the result, the negation, is a determinate negation, it has a content. It is a new concept but one higher and richer than the preceding --- richer because it negates or opposes the preceding and therefore contains it, and it contains even more than that, for it is the unity of itself and its opposite'' \citep[p.~33]{hegel2010}. We retain Assumptions~\ref{as:decisive} and~\ref{as:genrich} as general architectural hypotheses below so that our downstream results apply to any admissible pair $(\mathcal{E}, \mathcal{M}_{n})$, while noting that trace-objectifying architectures guarantee the existence of such extensions by Proposition~\ref{prop:tracesurplus} and generate them operationally via Proposition~\ref{prop:tracefallible}.

\begin{assumption}[Decisive Extension]\label{as:decisive}
At every level $n$ at which recursive extension is triggered, the decisive surplus is non-empty,
\begin{equation}\label{eq:asdecisive}
  \mathrm{Surp}^{\mathrm{dec}}(O_{n}) \;\neq\; \emptyset ,
\end{equation}
and the meta-level selection operator restricts its choices to decisive successors, $\mathcal{M}_{n}\bigl(\mathrm{Surp}(O_{n})\bigr) \subseteq \mathrm{Surp}^{\mathrm{dec}}(O_{n})$.
\end{assumption}

\paragraph{Meta-level selection.}

A \emph{meta-level operator}
\begin{equation}\label{eq:metasel}
  \mathcal{M}_{n} \colon 2^{\mathcal{E}(\mathcal{F}_{n})}
  \;\rightharpoonup\; 2^{\mathcal{E}(\mathcal{F}_{n})} \setminus \{\emptyset\},
  \qquad
  \mathcal{M}_{n}(X) \subseteq X,
\end{equation}
narrows the admissible successors to those that the architecture is prepared to adopt. In doing so, it may compare candidates, evaluate their consequences, or prune those that fail its structural criteria; what it returns is not a single successor but the set that survives. As with \eqref{eq:top}, selection is relational rather than functional: the architecture determines which continuations are available, but not which is taken. The recursive transition is then
\begin{equation}\label{eq:transition}
  \mathcal{F}_{n}
  \;\xrightarrow{\ \mathcal{E}\ }\;
  \mathrm{Surp}(O_{n})
  \;\xrightarrow{\ \mathcal{M}_{n}\ }\;
  \mathcal{F}_{n+1} \in \mathcal{M}_{n}\bigl(\mathrm{Surp}(O_{n})\bigr).
\end{equation}
The essential point is that the meta-level does not choose among answers produced within a fixed system. It operates on the space of possible successor systems themselves. Because the architecture is homoiconic (\S\ref{sec:laws}), the agent's teleological partition $\Phi_{n} = (\Phi^{+}_{n}, \Phi^{-}_{n}, \Phi^{\circ}_{n})$ of \eqref{eq:partition} is itself represented within $L_{n}$, making $\mathcal{M}_{n}$ the formal realization of the teleological transformation operator $\mathcal{T}$ of \eqref{eq:top}. Under Assumption~\ref{as:decisive}, $\mathcal{M}_{n}$ selects a decisive successor $\mathcal{F}_{n+1} \in \mathrm{Surp}^{\mathrm{dec}}(O_{n})$ whose extended vocabulary $L_{n+1} \setminus \tau_{n}(L_{n})$ can articulate a successor purpose $\phi_{n+1}$ with a non-contained semantic domain $D_{\phi_{n+1}} \not\subseteq D_{H_{n}}$ in \eqref{eq:succ}, without the Boolean closure constraint \eqref{eq:gcond} that bounded $\mathcal{T}$ over a fixed language.

Where the teleological model of \S\ref{sec:recursive_teleology} expresses novelty set-theoretically via $D_{\phi_{n+1}} \not\subseteq D_{H_{n}}$ in \eqref{eq:succ}, the formal model requires an algorithmic measure of the information introduced when $\mathcal{M}_{n}$ commits to $\mathcal{F}_{n+1}$ over competing branches in $\mathrm{Surp}(O_{n})$. That stronger requirement is stated separately as Assumption~\ref{as:novelty} in \S\ref{sec:independent_semantic_extension}. The two conditions play the same architectural role: each ensures that successive extensions introduce genuinely new discriminations rather than relabelled repetitions of prior structure.

The decomposition $I(O_{n}) = I^{\mathrm{loc}}(O_{n}) \;\dot\cup\; I^{\mathrm{inc}}(O_{n})$ together with \eqref{eq:formcontent} and \eqref{eq:transition} marks the exact boundary between deductively certified self-improvement and co-recursive self-extension. In a Gödel machine \citep{schmidhuber2006goedel}, the proof searcher cannot know in advance which candidate code transformations have provable utility; it must run an object-level search that generates a derivation trace $\pi_{0,t}(x)$, but it is licensed to execute a rewrite only when that trace closes within $I^{\mathrm{loc}}(O_{0})$ into a theorem derived from the initial axiom scheme $\mathcal{F}_{0}$. In his discussion of the limitations of Gödel machines, Schmidhuber notes that ``the fundamental limitations are closely related to those first identified by Gödel's celebrated paper on self-referential formulae'' \citep[p.~6]{schmidhuber2006goedel}, because any formal system encompassing arithmetic ``is either flawed or allows for unprovable but true statements. Hence, even a Gödel machine with unlimited computational resources must ignore those self-improvements whose effectiveness it cannot prove, e.g., for lack of sufficiently powerful axioms in A'' \citep[p.~6]{schmidhuber2006goedel}. Confronted with a candidate modification whose utility lies in $I^{\mathrm{inc}}(O_{0})$, the machine ``cannot profit from a hypothetical useful self-improvement whose utility is undecidable, and will therefore simply ignore it'' \citep[p.~24]{schmidhuber2006goedel}---discarding the unclosed derivation trace $\pi_{0,t}(x)$ because it yields no object-level theorem. In the present framework, that limitation is inverted. By Propositions~\ref{prop:tracesurplus} and~\ref{prop:tracefallible}, an unclosed proper reduction trace on $I^{\mathrm{inc}}(O_{n})$ is not an empty failure at which self-modification halts, but a determinate negation whose operational form $\pi_{n,t}(x)$ is reflectively reified via \eqref{eq:formcontent} into the semantic content of candidate extensions \citep[pp.~2, 51]{hegel1977phenomenology}, at least one pair of which is guaranteed to lie in $\mathrm{Surp}^{\mathrm{dec}}(O_{n})$. Because no finite search within $\mathcal{F}_{n}$ can decide in advance which open premise $b \in \mathcal{B}_{n,t}(x)$ lies in $I^{\mathrm{inc}}(O_{n})$ (Proposition~\ref{prop:tracefallible}) and no element of $\mathrm{Surp}^{\mathrm{dec}}(O_{n})$ can be certified as superior by a proof within $\mathcal{F}_{n}$ without collapsing $\mathcal{F}_{n+1}$ back into the deductive closure of $\Gamma_{n}$, the operator $\mathcal{M}_{n}$ cannot be reduced to theorem proving in $\mathcal{F}_{n}$; it mediates across discontextural regimes \citep[p.~15]{gunther1979life} by filtering out inconsistent candidate branches and making a non-deductive commitment to a genuinely new formal contexture.

\paragraph{Successor incompleteness.}

Because admissibility requires regularity, \S\ref{sec:formal_system_cognitive_state_and_incompleteness} applies verbatim to $\mathcal{F}_{n+1}$. There is therefore a sentence $G_{n+1} \in L_{n+1}$ with
\begin{equation}\label{eq:godelsucc}
  \mathcal{F}_{n+1} \nvdash G_{n+1}
  \qquad\text{and}\qquad
  \mathcal{F}_{n+1} \nvdash \neg G_{n+1},
\end{equation}
and hence $I^{\mathrm{inc}}(O_{n+1}) \neq \emptyset$.

Resolving an ambiguity at one level does not eliminate incompleteness; it just relocates it. Schematically, the architecture moves from an undecidable sentence through a successor system to a fresh undecidable sentence, and this is the fundamental recursive structure of the theory: a sufficiently expressive system can construct a stronger system that transcends some of its limitations, which in turn is subject to its own limitations. The consequences of iterating this construction are developed in \S\ref{sec:successor_incompleteness}.

The central question is therefore no longer whether a single formal system can become complete. It is whether a cognitive architecture can sustain the construction, preservation, evaluation, and selection of successor systems without reducing the process to a fixed formal space---the question \S\ref{sec:representational_limits_openended_cognitive_architectures}answers negatively for fixed architectures.

\subsection{Temporal Recursive Dynamics}\label{sec:temporal_recursive_dynamics}

The recursive structure introduced above unfolds in time. Let $\mathcal{F}_{n,t}$ denote the inferential system available at recursive level $n$ and temporal step $t$, and $O_{n,t}$ the corresponding cognitive state. Two kinds of transition must be distinguished.

A \emph{local transformation} operates within the current formal system,
\begin{equation}\label{eq:local}
  \mathcal{F}_{n,t} \longrightarrow \mathcal{F}_{n,t+1},
  \qquad L_{n,t+1} = L_{n,t}, \quad \Gamma_{n,t+1} = \Gamma_{n,t},
\end{equation}
changing only the cognitive state. By \eqref{eq:sound}, such a transition can only extend $\operatorname{dom}(O_{n,t})$ within the decidable fragment, so
\begin{equation}\label{eq:localeffect}
  I^{\mathrm{loc}}(O_{n,t+1}) \subseteq I^{\mathrm{loc}}(O_{n,t}),
  \qquad
  I^{\mathrm{inc}}(O_{n,t+1}) = I^{\mathrm{inc}}(O_{n,t}).
\end{equation}
Local work consumes the first part of the unresolved region and cannot touch the second. This is the formal content of the distinction drawn at \eqref{eq:decomp}.

A \emph{recursive extension} constructs a new inferential system,
\begin{equation}\label{eq:extend}
  \mathcal{F}_{n,t} \longrightarrow \mathcal{F}_{n+1,0},
\end{equation}
via \eqref{eq:transition}. Here, $L$ and $\Gamma$ themselves change, and \eqref{eq:localeffect} no longer applies: the successor's unresolved region is a region of a different language, related to its predecessor's only through $\tau_{n}$.

The temporal trajectory of a system is therefore
\begin{equation}\label{eq:traj}
  \mathrm{Traj}(\mathcal{F}_{0}) =
  (\mathcal{F}_{0,0}, \mathcal{F}_{0,1}, \dots,
   \mathcal{F}_{1,0}, \mathcal{F}_{1,1}, \dots,
   \mathcal{F}_{2,0}, \dots),
\end{equation}
alternating runs of local transformation with extension steps.

\paragraph{Generative richness.}

For general transformation repertoires $\mathcal{E}(\mathcal{F}_{n,t})$ that do not explicitly require trace-objectification \eqref{eq:formcontent}, we state the branching requirement proved for trace-reifying systems in Proposition~\ref{prop:tracesurplus} as a general architectural condition.

\begin{assumption}[Generative Richness]\label{as:genrich}
At every level $n$ and step $t$ at which extension is attempted,
\begin{equation}\label{eq:genrich}
  \bigl|\mathrm{Surp}(O_{n,t})\bigr| \;\geq\; 2 .
\end{equation}
\end{assumption}

\noindent
When $\mathcal{E}(\mathcal{F}_{n,t})$ is left unrestricted, \eqref{eq:genrich} is not implied by \eqref{eq:irreduce} alone, since a system without reflective trace reification could face an unresolved proposition while possessing at most one admissible successor; when $\mathcal{E}(\mathcal{F}_{n,t})$ includes \eqref{eq:formcontent}, Proposition~\ref{prop:tracesurplus} guarantees \eqref{eq:genrich} constructively from any unclosed proof search on $I^{\mathrm{inc}}(O_{n,t})$. It is what \S\ref{sec:trajectory_complexity} needs for the space of trajectories to carry a binary-tree structure, and what \S\ref{sec:from_understanding_to_speculative_intelligence} restates at the level of an individual agent within a collective.

Recursive extension does not imply success. A successor may resolve an ambiguity present in $\mathcal{F}_{n}$, leave it unresolved, reinterpret the relevant structure, or introduce new ambiguities of its own. Recursion, therefore, produces exploration rather than certainty, and the trajectory \eqref{eq:traj} should not be read as monotone progress toward a final system.

\subsection{Local Stability and Global Recursive Extension}\label{sec:local_stability_and_global_recursive_extension}

The preceding distinction can be sharpened by assigning a state space to each inferential regime. Let
\begin{equation}\label{eq:statespace}
  (X_{n}, d_{n}),
  \qquad
  X_{n} = \{\, O \colon L_{n} \rightharpoonup \mathcal{V}_{n}
    \mid O \text{ satisfies } \eqref{eq:sound} \,\},
\end{equation}
be a metric space, where $X_{n}$ is the set of cognitive states admissible within regime $\mathcal{F}_{n}$ and
\begin{equation}\label{eq:metric}
  d_{n} \colon X_{n} \times X_{n} \to \mathbb{R}_{\geq 0}
\end{equation}
measures the semantic distance between states. Note that $X_{n}$ is non-trivial precisely because \eqref{eq:sound} was not strengthened to a biconditional: i.e., had commitment been required wherever proof is available, $X_{n}$ would contain a single point and the dynamics below would be empty.

A local transformation is a map $f^{(n)}_{i} \colon X_{n} \to X_{n}$. For some applications, it is useful to require these to be contractive.

\begin{assumption}[Local Contractivity]\label{as:contract}
For each relevant transformation $f^{(n)}_{i}$ there exists $0 \leq c_{n} < 1$ with
\begin{equation}\label{eq:contract}
  d_{n}\bigl(f^{(n)}_{i}(x), f^{(n)}_{i}(x')\bigr)
  \;\leq\; c_{n}\, d_{n}(x, x')
  \qquad \forall\, x, x' \in X_{n}.
\end{equation}
\end{assumption}

\noindent
This expresses a local engineering principle: i.e., within a fixed semantic regime, reasoning reduces uncertainty and converges toward a stable interpretation. If the family
\begin{equation}\label{eq:ifs}
  \mathcal{H}_{n} = \{f^{(n)}_{1}, \dots, f^{(n)}_{k_{n}}\}
\end{equation}
consists of contractive maps on a complete metric space, the corresponding iterated function system admits a unique compact attractor $\Lambda_{n} \subseteq X_{n}$ under the standard conditions of Hutchinson's theorem.

The significance of this result for the present theory is limited but important. It establishes that local stability and open-ended global evolution are not contradictory requirements. Local convergence has the form
\begin{equation}\label{eq:converge}
  X_{n} \longrightarrow \Lambda_{n},
\end{equation}
while recursive extension has the form
\begin{equation}\label{eq:regimechange}
  X_{n} \longrightarrow X_{n+1}.
\end{equation}
The two operate at different structural levels. Convergence to $\Lambda_{n}$ substantially exhausts $I^{\mathrm{loc}}$ but, by \eqref{eq:localeffect}, leaves $I^{\mathrm{inc}}$ intact---so reaching the attractor is the condition under which \eqref{eq:transition} becomes the only remaining move. In this sense, the attractor is the formal counterpart of the $\varepsilon$-saturation criterion \eqref{eq:sat} of \S\ref{sec:recursive_teleology}, with the difference that its residue is guaranteed non-empty by \eqref{eq:irreduce} rather than assumed.

\subsection{Recursive Extension}\label{sec:recursive_extension}

We now state the successor operation in general form.

\begin{definition}[Recursive Extension Operator]\label{def:rext}
A \emph{recursive extension} at level $n$ is a pair
\begin{equation}\label{eq:rext}
  \mathcal{R}_{n} = (\mathcal{F}_{n+1}, \tau_{n}),
  \qquad
  \mathcal{F}_{n+1} \in \mathcal{E}(\mathcal{F}_{n}),
  \qquad
  \tau_{n} \colon L_{n} \to L_{n+1},
\end{equation}
selected by \eqref{eq:transition}, where $\mathcal{F}_{n+1}$ satisfies regularity and $\tau_{n}$ satisfies reflective adequacy in the sense of \S\ref{sec:reflective_ext_and_semantic_surplus}. The operator may modify the language, axioms, inference rules, representational structures, or meta-level resources of the system.
\end{definition}

\noindent
We write
\begin{equation}\label{eq:rarrow}
  \mathcal{F}_{n} \xrightarrow{\ \mathcal{R}_{n}\ } \mathcal{F}_{n+1}
\end{equation}
for the resulting transition. $\mathcal{R}_{n}$ carries selected structures forward while potentially transforming their interpretation. What reflective adequacy secures is that $L_{n+1}$ can express statements about $\mathcal{F}_{n}$, so that the predecessor becomes an \emph{object of reasoning} for the successor rather than merely a stage it has passed through. This is the formal counterpart of recursive reflection, and the property on which the memory requirements of \S\ref{sec:collective_memory_extended_cognitive_space} depend.

\subsection{Successor Incompleteness}\label{sec:successor_incompleteness}

The significance of recursive extension follows directly from \S\ref{sec:formal_system_cognitive_state_and_incompleteness}. Suppose $\mathcal{F}_{n}$ is regular and $\mathcal{R}_{n}$ constructs a successor $\mathcal{F}_{n+1}$; admissibility requires $\mathcal{F}_{n+1}$ to be regular in turn. Then by \eqref{eq:godelsucc} there is $G_{n+1} \in L_{n+1}$ with
\begin{equation}\label{eq:godeln1}
  \mathcal{F}_{n+1} \nvdash G_{n+1}
  \qquad\text{and}\qquad
  \mathcal{F}_{n+1} \nvdash \neg G_{n+1}.
\end{equation}
Iterating, the architecture generates
\begin{equation}\label{eq:chain}
  \mathcal{F}_{0}
  \xrightarrow{\ \mathcal{R}_{0}\ } \mathcal{F}_{1}
  \xrightarrow{\ \mathcal{R}_{1}\ } \mathcal{F}_{2}
  \xrightarrow{\ \mathcal{R}_{2}\ } \cdots
\end{equation}
with, at every level,
\begin{equation}\label{eq:persist}
  G_{n} \in I^{\mathrm{inc}}(O_{n})
  \qquad\text{and}\qquad
  I^{\mathrm{inc}}(O_{n}) \neq \emptyset .
\end{equation}
Recursive extension, therefore, does not eliminate incompleteness. It produces a new level at which incompleteness arises again. This is not a defect of the construction but of its engine: i.e., by \eqref{eq:persist}, the recursive pressure of \S\ref{sec:formal_system_cognitive_state_and_incompleteness} is renewed at every level, so the architecture never reaches a state in which extension becomes unmotivated.

Two limits of this result should be stated plainly. First, \eqref{eq:persist} guarantees that each level has \emph{something} undecidable, whereas ensuring that the successor decides what its predecessor could not relies on the decisive surplus \eqref{eq:decisive}, which requires reflective trace reification (Proposition~\ref{prop:tracesurplus}, stated generally in Assumption~\ref{as:decisive}) rather than bare extensional incompleteness alone. Second, \eqref{eq:chain} establishes that the sequence can be continued indefinitely, not that it must be, and not that its members differ in any way a reader would call progress. The stronger condition that successive extensions carry irreducible novelty is introduced separately as Assumption~\ref{as:novelty} in \S\ref{sec:independent_semantic_extension}. The essential property of the architecture is accordingly not that it reaches a complete state, but that it can construct successor systems in which unresolved questions can be reconsidered from a different representational position.

\subsection{Local Convergence and Global Open-Endedness}\label{sec:local_convergence_and_global_openendedness}

The distinction between local dynamics and recursive extension allows open-endedness to be defined without requiring unbounded expansion inside any single semantic space.

At each level, the system may operate within a bounded or locally contractive regime, converging by \eqref{eq:converge} to an attractor $\Lambda_{n}$. Yet the architecture may continue to generate successor regimes,
\begin{equation}\label{eq:regimeseq}
  (X_{0}, d_{0}) \longrightarrow (X_{1}, d_{1})
  \longrightarrow (X_{2}, d_{2}) \longrightarrow \cdots
\end{equation}
Open-endedness, therefore, does not require that any individual semantic space be unbounded. It requires that the process of constructing successor spaces not be restricted to a fixed finite hierarchy---the condition invoked at \eqref{eq:phigrow} in the teleological model, now without the closure constraint \eqref{eq:gcond} that bounded it there.

This yields the central distinction of the section:
\begin{equation}\label{eq:localglobal}
  \text{local convergence} \;\neq\; \text{global termination}.
\end{equation}
The inequality is not merely definitional. Local convergence means $I^{\mathrm{loc}}(O_{n}) = \emptyset$, attained at the attractor $\Lambda_{n}$. Global termination would additionally require $I^{\mathrm{inc}}(O_{n}) = \emptyset$, which \eqref{eq:irreduce} forbids for any regular system. A system can therefore be stable with respect to its current regime while remaining capable of constructing new ones, and the stability is not evidence that the process has ended.

Relation~\eqref{eq:localglobal} directly resolves the ``stopping problem'' identified by \citet[p.~166]{he2026cuv} for $(C, U, V)$ agents. In their formalization, an agent's evolution converges, $\lim_{t \to \infty} (C_{t}, U_{t}, V_{t}) = (C^{\ast}, U^{\ast}, V^{\ast})$, once the system attains ``cognitive completeness''---where the parse-graph parameters $\theta_{\mathrm{PG}}$ ``can express all reasonable human cognition and possible actions for the given scene'' \citep[p.~174]{he2026cuv}---and ``task completeness''---where the agent has ``the capability to find all possible target PGs that can increase the values in its $V$-dimension, as well as the capability to determine a path from the agent's current PG to the future target PG'' \citep[p.~174]{he2026cuv}. At that point, ``when an agent's cognitive architecture determines that the agent's current skills have reached the upper limit of what the agent can master, or that the agent's current value functions have reached the optimal level the agent can achieve, the cognitive architecture cannot further orchestrate the learning process, leading to the convergence of learning'' \citep[p.~175]{he2026cuv}. In the terms of \S\ref{sec:formal_system_cognitive_state_and_incompleteness}--\S\ref{sec:local_stability_and_global_recursive_extension}, cognitive and task completeness describe the exhaustion of the decidable unresolved region $I^{\mathrm{loc}}(O_{n}) \to \emptyset$ as the trajectory approaches the local attractor $\Lambda_{n} \subseteq X_{n}$ in \eqref{eq:converge}. That convergence becomes a terminal stopping problem in $(C, U, V)$ only because cognitive completeness is defined relative to a fixed function class $\theta_{\mathrm{PG}}$ and a fixed value specification $V_{\mathrm{select}}$. Within any sufficiently expressive regime $\mathcal{F}_{n}$, however, exhausting $I^{\mathrm{loc}}(O_{n})$ at $\Lambda_{n}$ leaves the undecidable region $I^{\mathrm{inc}}(O_{n})$ intact by \eqref{eq:localeffect}. What \citet{he2026cuv} characterize as the terminal boundary of learning within a given $(C, U, V)$ space is here the trigger for the co-recursive transition $X_{n} \to X_{n+1}$ in \eqref{eq:regimechange}, provided the architecture is governed by a meta-teleology that treats saturation at $\Lambda_{n}$ as the signal to extend the representational and teleological framework itself.

We can make this coupling between local convergence and global co-recursion explicit by formalizing the \textit{a priori} meta-teleological trigger introduced in \S\ref{sec:laws} and \S\ref{sec:recursive_teleology}.

\begin{definition}[A Priori Meta-Teleological Trigger]\label{def:metatelos}
Let $\bigl((X_{n}, d_{n}), \mathcal{H}_{n}\bigr)$ be a locally contractive cognitive regime satisfying Assumption~\ref{as:contract} with compact attractor $\Lambda_{n} \subseteq X_{n}$ driven by the active first-order teleology $\Phi^{+}_{n}$. Given a saturation threshold $\varepsilon > 0$, an \emph{a priori meta-teleological trigger} is a second-order transition rule that monitors the metric proximity $d_{n}(O_{n,t}, \Lambda_{n}) := \inf_{y \in \Lambda_{n}} d_{n}(O_{n,t}, y)$ and, at the first step $t_{n}^{\ast}$ for which
\begin{equation}\label{eq:metatrigger}
  d_{n}\bigl(O_{n,t_{n}^{\ast}}, \Lambda_{n}\bigr) \;<\; \varepsilon ,
\end{equation}
invokes the meta-level selection operator $\mathcal{M}_{n}$ on the candidate extensions \eqref{eq:formcontent} generated from the unclosed proper reduction traces $\pi_{n,t_{n}^{\ast}}(x)$ of $x \in I(O_{n,t_{n}^{\ast}})$---which by \eqref{eq:irreduce}, \eqref{eq:localeffect}, and Proposition~\ref{prop:tracefallible} are guaranteed to contain at least one genuine decisive pair in $\mathrm{Surp}^{\mathrm{dec}}(O_{n,t_{n}^{\ast}})$ alongside any latent local branches---to select a decisive extension $\mathcal{F}_{n+1,0} \in \mathcal{M}_{n}\bigl(\mathrm{Surp}^{\mathrm{dec}}(O_{n,t_{n}^{\ast}})\bigr)$ under Assumption~\ref{as:decisive} (or Propositions~\ref{prop:tracesurplus}--\ref{prop:tracefallible}), initiating the regime transition $(X_{n}, d_{n}) \to (X_{n+1}, d_{n+1})$.
\end{definition}

\begin{proposition}[Escape from Local Convergence]\label{prop:stopping}
Let $\mathcal{A}$ be a cognitive architecture whose local dynamics on each regime $(X_{n}, d_{n})$ satisfy Inferential Soundness (Assumption~\ref{as:sound}) and Local Contractivity (Assumption~\ref{as:contract}) with contraction factor $0 \leq c_{n} < 1$.
\begin{enumerate}
  \item \emph{(Absorption under intra-regime selection.)} Any architecture whose self-modification rule requires deductive certification within $\mathcal{F}_{n}$ (acting only on $I^{\mathrm{loc}}(O_{n,t})$, as in a Gödel machine \citep{schmidhuber2006goedel}) or only reweights objectives within the fixed space $X_{n}$ (as in $(C, U, V)$ selector convergence \citep{he2026cuv}) satisfies $d_{n}(O_{n,t}, \Lambda_{n}) \leq c_{n}^{t}\, d_{n}(O_{n,0}, \Lambda_{n}) \to 0$ as $t \to \infty$ and remains permanently confined to $X_{n}$.
  \item \emph{(Unbounded regime succession under meta-teleological triggering.)} If $\mathcal{A}$ satisfies Decisive Extension (Assumption~\ref{as:decisive}) and is equipped with the \emph{a priori} meta-teleological trigger of Definition~\ref{def:metatelos} for some $\varepsilon > 0$, then at every level $n \geq 0$ the trigger condition \eqref{eq:metatrigger} is reached in at most
  \begin{equation}\label{eq:triggerbound}
    t_{n}^{\ast} \;\leq\; \max\Bigl\{0,\; \Bigl\lfloor \frac{\ln\bigl(d_{n}(O_{n,0}, \Lambda_{n}) / \varepsilon\bigr)}{\ln(1/c_{n})} \Bigr\rfloor + 1 \Bigr\}
  \end{equation}
  local steps, at which point $I^{\mathrm{inc}}(O_{n,t_{n}^{\ast}}) \neq \emptyset$ by \eqref{eq:irreduce} and \eqref{eq:localeffect}, inducing an infinite sequence of distinct regimes $(X_{0}, d_{0}) \to (X_{1}, d_{1}) \to (X_{2}, d_{2}) \to \cdots$ in which no local attractor $\Lambda_{n}$ is a global halting state.
\end{enumerate}
\end{proposition}

\begin{proof}
For (1), every local transformation $f^{(n)}_{i} \in \mathcal{H}_{n}$ preserves $L_{n}$ and $\Gamma_{n}$ by \eqref{eq:local} and acts solely on $I^{\mathrm{loc}}(O_{n,t})$ by \eqref{eq:localeffect}, while leaving $I^{\mathrm{inc}}(O_{n,t})$ invariant. By Assumption~\ref{as:contract} and Hutchinson's theorem, the induced trajectory on $(X_{n}, d_{n})$ contracts toward $\Lambda_{n}$ at rate $c_{n}^{t}$, and because no transition rule is licensed to adopt an element of $\mathrm{Surp}^{\mathrm{dec}}(O_{n,t})$ whose validity lies in $I^{\mathrm{inc}}(O_{n,t})$, the state remains in $X_{n}$ for all $t \geq 0$. 

For (2), contractivity implies $d_{n}(O_{n,t}, \Lambda_{n}) \leq c_{n}^{t}\, d_{n}(O_{n,0}, \Lambda_{n})$, which falls below $\varepsilon$ no later than the bound $t_{n}^{\ast}$ in \eqref{eq:triggerbound}. At step $t_{n}^{\ast}$, regularity of $\mathcal{F}_{n}$ guarantees $I^{\mathrm{inc}}(O_{n,t_{n}^{\ast}}) = I^{\mathrm{inc}}(O_{n,0}) \neq \emptyset$ by \eqref{eq:irreduce} and \eqref{eq:localeffect}, and Assumption~\ref{as:decisive} guarantees $\mathrm{Surp}^{\mathrm{dec}}(O_{n,t_{n}^{\ast}}) \neq \emptyset$. Definition~\ref{def:metatelos} then executes $\mathcal{F}_{n,t_{n}^{\ast}} \to \mathcal{F}_{n+1,0}$, which by regularity of $\mathcal{F}_{n+1,0}$ and \eqref{eq:persist} renews the hypotheses at level $n+1$, continuing the sequence indefinitely by induction on $n$.
\end{proof}

The mathematical question is consequently shifted. It is not whether a single state space can contain every future cognitive state, but whether an architecture can sustain an indefinitely extensible sequence of representational systems---and, as \S\ref{sec:representational_limits_openended_cognitive_architectures}shows, whether any \emph{fixed} architecture can do so at all.

\section{Representational Limits of Open-Ended Cognitive Architectures}\label{sec:representational_limits_openended_cognitive_architectures}

The preceding section establishes a recursive architecture in which a formal system constructs successor systems when its current representational resources prove insufficient. We now ask what this process costs in representational terms.

The question is whether a fixed internal description and a fixed transition mechanism can completely represent an indefinitely self-extending architecture. This is distinct from the question of computational scale. A system may possess arbitrarily large computational resources while remaining within a fixed representational framework; conversely, a system may recursively extend its framework without any individual regime being unbounded. The issue is not how much computation is available at a given level, but whether the information required to specify the successive regimes can remain bounded.

\subsection{Semantic Description Complexity}

Let $K(x)$ denote the Kolmogorov complexity of an object $x$ with respect to a fixed universal prefix-free Turing machine, and $K(x \mid y)$ the corresponding conditional complexity. We assume throughout that formal systems are encoded as finite binary strings under a fixed effective encoding; this is possible because each $\mathcal{F}_{n}$ is effectively axiomatized by \S\ref{sec:formal_system_cognitive_state_and_incompleteness}.

For the sequence of systems generated by \eqref{eq:chain}, define the \emph{accumulated history}
\begin{equation}\label{eq:history}
  H_{n} = (\mathcal{F}_{0}, \mathcal{F}_{1}, \dots, \mathcal{F}_{n}).
\end{equation}
The quantity $K(H_{n})$ measures the minimum algorithmic information required to specify the complete semantic history up to level $n$.

The distinction between $K(H_{n})$ and $K(\mathcal{F}_{n})$ matters. The complexity of an individual successor need not exceed that of its predecessor: i.e., a successor may contain a highly compressed representation of information previously distributed across several structures, so that $K(\mathcal{F}_{n+1}) < K(\mathcal{F}_{n})$ is entirely possible. What matters for open-ended evolution is whether the \emph{accumulated} history can remain algorithmically bounded.

\subsection{Independent Semantic Extension}\label{sec:independent_semantic_extension}

Let $\gamma > 0$ be a constant for which the chain rule for prefix Kolmogorov
complexity holds, for all $n \geq 2$, in the form
\begin{equation}\label{eq:chainconst}
  K(x, y) \;\geq\; K(x) + K(y \mid x) - \gamma \log_{2} n ,
\end{equation}
where $n$ bounds the lengths of the objects involved; the additive constant of
the usual statement is absorbed into $\gamma \log_{2} n$.

Before stating the algorithmic novelty condition, we note how the architectural postulates of \S\ref{sec:recursive_teleology} and \S\ref{sec:reflective_ext_and_semantic_surplus} give rise to conditional description complexity at each extension step. When the meta-level operator selects a successor $\mathcal{F}_{n+1} \in \mathcal{M}_{n}\bigl(\mathrm{Surp}^{\mathrm{dec}}(O_{n})\bigr)$, three structural features distinguish $\mathcal{F}_{n+1}$ from a deductive continuation of $H_{n}$. First, by Decisive Extension (Assumption~\ref{as:decisive}), $\mathcal{F}_{n+1}$ settles at least one frontier proposition $x_{n} \in I^{\mathrm{inc}}(O_{n})$ that is formally independent of $\Gamma_{n}$ and therefore not deductively determined by $H_{n}$. Second, by Generative Richness (Assumption~\ref{as:genrich}), $\mathrm{Surp}(O_{n})$ offers multiple competing continuations at each level. Third, by the non-monotone teleological succession condition \eqref{eq:succ}, the extended vocabulary $L_{n+1} \setminus \tau_{n}(L_{n})$ is selected in response to newly discriminated environmental signals $D_{\phi_{n+1}} \setminus D_{H_{n}} \neq \emptyset$ (or, in the collective setting of \S\ref{sec:distributed_semantic_surplus}, cross-agent polycontextural inputs) that were excluded from every prior domain in $H_{n}$. In a closed deductive rewriter such as a Gödel machine \citep{schmidhuber2006goedel}, where axioms are replaced only by ``theorems derivable from the original axioms'' \citep[p.~17]{schmidhuber2006goedel}, the successor at proof-search step $t_{n}$ is a computable function of $\mathcal{F}_{0}$ and $t_{n}$, so that $K(\mathcal{F}_{n+1} \mid H_{n}) \leq O(\log n)$. By contrast, whenever the commitment $\mathcal{F}_{n+1} \in \mathcal{M}_{n}\bigl(\mathrm{Surp}^{\mathrm{dec}}(O_{n})\bigr)$ encodes genuinely new semantic discrimination that no fixed finite program can compute from $H_{n}$ alone, its conditional complexity $K(\mathcal{F}_{n+1} \mid H_{n})$ exceeds the logarithmic prefix-coding overhead $\gamma \log_{2} n$ of \eqref{eq:chainconst} by a positive margin $c > 0$.

\begin{assumption}[Persistent Algorithmic Novelty]\label{as:novelty}
There is a constant $c > 0$ such that, for every level $n$,
\begin{equation}\label{eq:novelty}
  K(\mathcal{F}_{n+1} \mid H_{n}) \;\geq\; c + \gamma \log_{2} n .
\end{equation}
\end{assumption}

\noindent
Applying \eqref{eq:chainconst} to the history and its successor,
\begin{equation}\label{eq:chainrule}
  K(H_{n+1}) \;=\; K(H_{n}, \mathcal{F}_{n+1})
  \;\geq\; K(H_{n}) + K(\mathcal{F}_{n+1} \mid H_{n}) - \gamma \log_{2} n ,
\end{equation}
so that under Assumption~\ref{as:novelty} the logarithmic terms cancel and
\begin{equation}\label{eq:growth}
  K(H_{n+1}) \;\geq\; K(H_{n}) + c ,
\end{equation}
giving, by induction,
\begin{equation}\label{eq:linear}
  K(H_{n}) \;\geq\; K(H_{0}) + cn
  \qquad\text{and}\qquad
  \boxed{\lim_{n \to \infty} K(H_{n}) = \infty} .
\end{equation}
The conclusion is not
\begin{equation}\label{eq:notthis}
  K(\mathcal{F}_{n+1}) > K(\mathcal{F}_{n}),
\end{equation}
which need not hold at any level, but the divergence in \eqref{eq:linear}: \textit{the accumulated history grows at least linearly in the number of extensions.}

\begin{remark}[Step-Wise vs.\ Amortized Algorithmic Novelty]
The additive term $\gamma \log_{2} n$ in \eqref{eq:novelty} is included solely so that the prefix chain rule \eqref{eq:chainconst} yields the clean step-wise increment $K(H_{n+1}) - K(H_{n}) \geq c$ at every single transition in \eqref{eq:growth}. For the impossibility theorem of \S\ref{sec:fixed_architecture_limitation}, however, neither step-wise uniformity nor the chain-rule constant $\gamma$ is needed. What Theorem~\ref{thm:fixed} actually requires is only that the accumulated history complexity $K(H_{n})$ eventually exceed the logarithmic self-delimiting index bound $2\log_{2} n$ of a fixed simulator. In particular, Theorem~\ref{thm:fixed} holds unchanged under the much weaker \emph{amortized novelty condition} (positive lower asymptotic density)
\begin{equation}\label{eq:amortized}
  \liminf_{n \to \infty} \frac{1}{n} K(H_{n}) \;\geq\; c \;>\; 0,
\end{equation}
and indeed under any super-logarithmic divergence $\lim_{n \to \infty} \bigl(K(H_{n}) - 2\log_{2} n\bigr) = \infty$. Under \eqref{eq:amortized}, arbitrarily long runs of recursive extensions may consolidate or compress prior structure with $K(\mathcal{F}_{n+1} \mid H_{n}) = O(1)$, provided that genuinely incompressible semantic discriminations enter intermittently with positive asymptotic density $c > 0$.
\end{remark}

\subsection{Trajectory Complexity}\label{sec:trajectory_complexity}

Two distinct sources of complexity should be kept apart.

First, the space of \emph{possible} trajectories. By Assumption~\ref{as:genrich}, every extension step offers at least two admissible continuations. If branching occurs at infinitely many levels and the choices are independent, the set of trajectories through \eqref{eq:traj} contains a subset order-isomorphic to the infinite binary tree, and hence a subset in bijection with Cantor space $\{0,1\}^{\mathbb{N}}$. Its cardinality is therefore $2^{\aleph_{0}}$.

Second, the complexity of a \emph{single realized} trajectory. By \eqref{eq:linear}, one such trajectory has finite prefixes of unbounded algorithmic complexity.

These are different claims, and neither implies the other. Uncountability concerns how many trajectories the architecture could take; unbounded complexity concerns how much information one of them contains. A system could have uncountably many possible trajectories, all of which are algorithmically trivial, or a unique trajectory of unbounded complexity.

\subsection{The Fixed-Architecture Limitation}\label{sec:fixed_architecture_limitation}

\begin{theorem}[Unbounded Representational Complexity of Open-Ended Trajectories]
\label{thm:fixed}
Let
\begin{equation}\label{eq:arch}
  \mathcal{A} = (\Sigma, \Theta, \sigma_{0})
\end{equation}
be a cognitive architecture with a fixed finite state description $\Sigma$, a
fixed computable transition mechanism $\Theta$, and a fixed initial condition
$\sigma_{0} \in \Sigma$. Then $\mathcal{A}$ cannot generate, from $\sigma_{0}$
alone, a semantic history $H_{0}, H_{1}, H_{2}, \dots$ satisfying
Assumption~\ref{as:novelty}.
\end{theorem}

\begin{proof}
Suppose it could. Since $\Theta$ and $\sigma_{0}$ are fixed, the state sequence
$\sigma_{t+1} = \Theta(\sigma_{t})$ is determined by the finite description
$(\Theta, \sigma_{0})$ together with the index $n$. A program that reads this
description, iterates $\Theta$ from $\sigma_{0}$, and outputs the first $n+1$
systems therefore specifies $H_{n}$ completely, so that
\begin{equation}\label{eq:bounded}
  K(H_{n}) \;\leq\; K(\Theta, \sigma_{0}) + 2\log_{2} n + O(1),
\end{equation}
where $2\log_{2} n$ encodes $n$ self-delimitingly and the additive constant
covers the fixed simulation program. Apart from that logarithmic term, the bound
is independent of $n$.

By Assumption~\ref{as:novelty} and \eqref{eq:linear}, however,
\begin{equation}\label{eq:unbounded}
  K(H_{n}) \;\geq\; K(H_{0}) + cn
\end{equation}
with $c > 0$. Combining \eqref{eq:bounded} and \eqref{eq:unbounded},
\begin{equation}\label{eq:collision}
  K(H_{0}) + cn \;\leq\; K(\Theta, \sigma_{0}) + 2\log_{2} n + O(1),
\end{equation}
whose left-hand side grows linearly in $n$ while its right-hand side grows
logarithmically. The inequality therefore fails for all sufficiently large $n$,
and no such $\mathcal{A}$ exists.
\end{proof}

\noindent
A fixed closed transition system cannot, from a fixed finite initial description
alone, generate an indefinitely self-extending semantic trajectory carrying
persistent algorithmic novelty. Moreover, inspection of \eqref{eq:collision} confirms that the contradiction depends only on asymptotic growth exceeding $2\log_{2} n$, and therefore holds verbatim when the step-wise condition of Assumption~\ref{as:novelty} is relaxed to the amortized density condition \eqref{eq:amortized}.

\subsection{What the Result Does and Does Not Establish}\label{sec:what_the_result_does_and_does_not_establish}

Theorem~\ref{thm:fixed} does not say that a fixed computer cannot represent an infinite or uncountable state space, nor does it deny that a program of fixed Kolmogorov complexity can generate open-ended complexity. In his work on algorithmic theories of everything, Schmidhuber \citep{schmidhuber2000ate} shows that a compact, fixed program can produce boundless apparent complexity in two distinct ways, and examining why neither yields open-ended intelligence clarifies the exact force of Assumption~\ref{as:novelty}.

First, a single history $x = (x_{1}, x_{2}, \dots)$ can be generated by ``a finite algorithm A that computes $x_{k+1}$ ($k \ge 1$) from $x_{k}$ and additional information $\text{noise}_{k}$'' \citep[p.~4]{schmidhuber2000ate} where $\text{noise}_{k}$ is ``calculated by invoking a finite pseudorandom generator subroutine'' \citep[p.~4]{schmidhuber2000ate}. To an observer inside such a process, states may appear perpetually novel. Yet as Schmidhuber notes of histories ``deterministically computed by a short algorithm,'' ``the algorithmic entropy (compare {[}98{]}) of the universe does not increase over time, because a finite program conveying a finite amount of information is responsible for everything'' \citep[p.~37]{schmidhuber2000ate}. In the terms of Theorem~\ref{thm:fixed}, such a generator is a fixed closed architecture $\mathcal{A} = (\Sigma, \Theta, \sigma_{0})$ whose history satisfies \eqref{eq:bounded}: its complexity is fixed from the outset, and no extension step introduces irreducible algorithmic information.

Second, Schmidhuber considers universal dovetailing algorithms such as \texttt{FAST}, which ``simultaneously compute all computable universes in an interleaving fashion that outputs each individual universe as quickly as its fastest algorithm running just by itself, save for a constant factor independent of the universe's size'' \citep[p.~5]{schmidhuber2000ate}. Because ``computing all universes with all possible types of physical laws tends to be much cheaper in terms of information requirements than computing just one particular, arbitrarily chosen one, because there is an extremely short algorithm that systematically enumerates and runs all computable universes, while most individual universes have very long shortest descriptions'' \citep[p.~40]{schmidhuber2000ate}, a program of fixed $O(1)$ complexity will eventually enumerate every finite prefix $H_{n}$ of every computable semantic trajectory alongside all other strings. Yet enumerating the entire branching tree of \S\ref{sec:trajectory_complexity} is not open-ended intelligence. As Schmidhuber himself remarks when noting that \texttt{FAST} outputs prefixes of Chaitin's incompressible $\Omega_{n}$ within $O(n2^{n})$ steps, it does so ``only because it also outputs many other strings besides $\Omega_{n}$ --- there is still no fast way of identifying $\Omega_{n}$ among all the outputs'' \citep[p.~30]{schmidhuber2000ate}. By the discrimination principle of \S\ref{sec:intelligence_begins}, a procedure that admits every continuation without exclusion ($\Phi^{-}_{n} = \emptyset$ in \eqref{eq:data-alt}) makes no distinction at all. Open-ended intelligence is not the unselective enumeration of all possible branches, which carries negligible algorithmic information precisely because nothing is rejected; it is the realization of a \emph{single selective trajectory} $H_{n} = (\mathcal{F}_{0}, \dots, \mathcal{F}_{n})$ through the branching surplus via successive commitments $\mathcal{F}_{k+1} = \mathcal{M}_{k}(\mathrm{Surp}(O_{k}))$. Identifying one such trajectory that satisfies Assumption~\ref{as:novelty} among all strings printed by a dovetailer requires specifying the sequence of selections at each level, and by \eqref{eq:linear} that specification carries at least $cn$ bits of algorithmic information that the fixed dovetailer does not possess.

What Theorem~\ref{thm:fixed} rules out is therefore specific: \textit{a fixed closed architecture cannot by itself select and realize a semantic history satisfying Assumption~\ref{as:novelty}}. The conflict is between \eqref{eq:linear} and \eqref{eq:bounded}, and it arises whenever each extension step commits to a distinction whose algorithmic information cannot be compressed into the prior history.

To sustain such a process, at least one of the following must therefore change over time:

\begin{enumerate}
  \item the representational language of the system;
  \item the transition rules governing it;
  \item the external information available to it;
  \item its own meta-level description.
\end{enumerate}

Recursive self-extension supplies precisely such a mechanism. Rather than requiring a single fixed architecture to encode every future semantic distinction, the architecture constructs successors \eqref{eq:chain}, distributing the representational burden across a sequence of formal systems rather than imposing it on a single fixed state space. Item (3) is the route taken in \S7: a bounded agent coupled to other agents receives information not derivable from its own history.

\subsection{Compression and Semantic Growth}

The result clarifies the relation between compression and semantic extension.

Compression seeks the smallest description of an existing semantic domain,
\begin{equation}\label{eq:compress}
  K(\text{representation}) \ \downarrow
\end{equation}
subject to preserving the relevant information or predictive behavior. Recursive semantic extension addresses a different problem: when the current system cannot adequately represent or resolve a relevant ambiguity, the architecture constructs a successor \eqref{eq:rarrow}, and the objective is not to minimize the complexity of the current representation but to enlarge the space of representable distinctions.

These are compatible. Under Assumption~\ref{as:novelty} the accumulated history diverges by \eqref{eq:linear}, while each individual $K(\mathcal{F}_{n})$ may be small and may even decrease---a system may compress aggressively within a regime while the sequence of regimes accumulates irreducible information. This yields the central distinction of the section:
\begin{equation}\label{eq:compvsext}
  \text{compression within a semantic space}
  \;\neq\;
  \text{extension of the semantic space itself}.
\end{equation}

\subsection{Relation to Open-Ended Intelligence}\label{sec:relation_to_openended_intelligence}

The result provides a formal reason why open-ended intelligence cannot be reduced to optimization within a fixed representational space. A system may become increasingly capable within a fixed language---improving predictions, compressing observations, optimizing parameters---without thereby acquiring the ability to construct a new regime. That ability requires \eqref{eq:rarrow} with a successor introducing distinctions not derivable from its predecessor, which is \eqref{eq:decisive} on the logical side and Assumption~\ref{as:novelty} on the algorithmic side.

The significance of recursive teleology becomes clearer here. The problem is not that an intelligent system must compute more. It must determine when its current regime is insufficient, generate a successor, and continue reasoning within it---the three operations that \S\ref{sec:recursive_teleology} assigned to saturation detection, $\mathcal{E}_{\Phi}$, and $\mathcal{T}$ respectively, and that \S\ref{sec:reflective_ext_and_semantic_surplus} reassigns to \eqref{eq:irreduce}, $\mathcal{E}$, and $\mathcal{M}_{n}$.

Theorem~\ref{thm:fixed} also sets up the difficulty \S\ref{sec:collective_intelligence_and_distributed_semantic_extension} addresses. Every physical system is finite, so no individual architecture can be the unbounded source of novelty that Assumption~\ref{as:novelty} requires. Either the assumption fails in real systems, or the novelty comes from outside the individual, which is the possibility that \S\ref{sec:collective_intelligence_and_distributed_semantic_extension} develops.

\section{Collective Intelligence and Distributed Semantic Extension}\label{sec:collective_intelligence_and_distributed_semantic_extension}

Theorem~\ref{thm:fixed} establishes that no fixed closed architecture generates a semantic history satisfying Assumption~\ref{as:novelty}. Every physical cognitive system is fixed in this sense: its memory, processing capacity, energy, and temporal horizon are bounded, and its transition mechanism is finite. The theorem, therefore, applies to any individual agent.

This limitation is not specific to artificial systems. No individual human possesses unbounded memory or representational capacity, yet human knowledge accumulates across generations through language, writing, institutions, and external records---a semantic structure exceeding the state space of any individual participant.

We therefore ask which of the escape routes listed in \S\ref{sec:what_the_result_does_and_does_not_establish} a collective takes. It is route (3): external information. A bounded agent embedded in a collective receives representations not derivable from its own history, and a persistent external memory accumulates what no agent retains. This section makes that precise. It does not evade Theorem~\ref{thm:fixed}; a collective of finitely many fixed architectures is itself a fixed architecture and falls under the same result.

\subsection{Collective Cognitive System}\label{sec:collective_cognitive_system}

Let $C = \{A^{1}, \dots, A^{N}\}$ be a finite collection of cognitive agents. Each $A^{k}$ carries its own inferential system $\mathcal{F}^{k}_{n,t}$, cognitive state $O^{k}_{n,t}$, and state space $X^{k}_{n}$ in the sense of \eqref{eq:statespace}. Agents need not share a language.

Communication requires translation rather than identity of representation. For agents $A^{i}, A^{j}$ let
\begin{equation}\label{eq:sharedlang}
  \sigma^{i \to j}_{n} \colon L^{i}_{n} \rightharpoonup L^{j}_{n}
\end{equation}
be a partial translation, defined on the fragment of $A^{i}$'s language that $A^{j}$ can interpret. Its domain
\begin{equation}\label{eq:common}
  L^{i,j}_{n} = \operatorname{dom}(\sigma^{i \to j}_{n}) \subseteq L^{i}_{n}
\end{equation}
is the \emph{common ground} between the two agents. Partiality is essential: $L^{i,j}_{n} = L^{i}_{n}$ would make the agents interchangeable, while $L^{i,j}_{n} = \emptyset$ would make them incommunicable. Neither extreme supports collective extension, as \S\ref{sec:diversity_and_common_ground} shows.

The collective state is not the union of individual states but the individual states together with the relations coupling them,
\begin{equation}\label{eq:collstate}
  \mathbb{X}_{n,t} =
  \bigl( \{X^{k}_{n}\}_{k=1}^{N}, \ \{\sigma^{i \to j}_{n}\}_{i \neq j} \bigr).
\end{equation}

\subsection{Distributed Semantic Surplus}\label{sec:distributed_semantic_surplus}

Each agent has a semantic surplus $\mathrm{Surp}(O^{k}_{n,t})$ in the sense of \eqref{eq:surplus}. The collective surplus is
\begin{equation}\label{eq:collsurplus}
  \mathrm{Surp}_{C}(n,t) = \bigcup_{k=1}^{N} \mathrm{Surp}(O^{k}_{n,t}).
\end{equation}
The cardinality of this union is not what makes a collective more than the sum of its agents; that is measured in \S\ref{sec:collective_generative_capacity}. The structurally important property is that one agent's surplus can \emph{become} another's unresolved region. If $A^{i}$ constructs a successor $\mathcal{F}' \in \mathrm{Surp}(O^{i}_{n,t})$ and communicates it, the translated content $\sigma^{i \to j}_{n}(\mathcal{F}')$ may contain propositions that $A^{j}$ can express but not decide,
\begin{equation}\label{eq:crossagent}
  \sigma^{i \to j}_{n}\bigl(\mathrm{Surp}(O^{i}_{n,t})\bigr)
  \cap I^{\mathrm{inc}}(O^{j}_{n,t}) \neq \emptyset,
\end{equation}
so that what one agent has resolved becomes a source of recursive pressure for another. By \eqref{eq:irreduce}, this pressure can then drive $A^{j}$ to construct successors of its own, which may in turn return to $A^{i}$ in transformed form.

\subsection{Collective Co-Recursion}\label{sec:collective_corecursion}

\emph{Collective co-recursion} is the process in which semantic transformations cross the boundaries of individual systems and generate successor states in others. For agents $A^{i}, A^{j}$ a collective transition has the form
\begin{equation}\label{eq:colltrans}
  I^{\mathrm{inc}}(O^{i}_{n,t})
  \ \xrightarrow{\ \mathcal{E}\ }\ \mathrm{Surp}(O^{i}_{n,t})
  \ \xrightarrow{\ \sigma^{i \to j}_{n}\ }\ I^{\mathrm{inc}}(O^{j}_{n,t})
  \ \xrightarrow{\ \mathcal{E}\ }\ \mathrm{Surp}(O^{j}_{n,t+1}),
\end{equation}
and a collective trajectory is a sequence
\begin{equation}\label{eq:colltraj}
  O^{k_{0}}_{n_{0},t_{0}} \rightsquigarrow O^{k_{1}}_{n_{1},t_{1}}
  \rightsquigarrow O^{k_{2}}_{n_{2},t_{2}} \rightsquigarrow \cdots
\end{equation}
in which the agent index may change at each step. The semantic process thereby becomes independent of the physical identity of any single agent. In other words, a transformation need not remain within the system that generated it.

\paragraph{What interaction alone does not achieve.}

Interaction between finitely many bounded agents does not by itself escape Theorem~\ref{thm:fixed}. If each $A^{k}$ has a finite state space $B^{k}$ and no state persists outside the agents, the joint state space $\prod_{k=1}^{N} B^{k}$ is finite, the collective is a fixed architecture in the sense of \eqref{eq:arch}, and \eqref{eq:bounded} applies to it unchanged. The collective trajectory \eqref{eq:colltraj} must then eventually revisit a joint state and cycle. Coupling redistributes cognitive work; it does not create algorithmic novelty.

What changes the situation is persistent external memory.

\begin{proposition}[Bounded Agents with Persistent Memory]\label{prop:collective}
Let $C$ be a collective in which
\begin{enumerate}
  \item each agent $A^{k}$ has bounded state space $B^{k}$;
  \item the translations \eqref{eq:sharedlang} are non-trivial and partial;
  \item there is a memory $\mathrm{Mem}_{C}$, external to every agent, whose contents persist independently of the agents' states, and whose size is not bounded in advance;
  \item agents may read from and write to $\mathrm{Mem}_{C}$.
\end{enumerate}
Then the collective semantic history is not bounded by the state space of any
individual agent, nor by $\prod_{k} B^{k}$.
\end{proposition}

\begin{proof}
The collective's state at time $t$ is the pair $\bigl((O^{1}_{t},\dots,O^{N}_{t}),  \mathrm{Mem}_{C}(t)\bigr)$. By Condition (3), the second component is not drawn from a fixed finite set, so the collective state space is not $\prod_{k} B^{k}$, and the cycling argument above does not apply. By Condition (2) and \eqref{eq:crossagent}, a successor constructed by $A^{i}$ and written to $\mathrm{Mem}_{C}$ may be read by $A^{j}$ and interpreted under a different inferential system, yielding by \eqref{eq:transition} successors unavailable within $\mathcal{F}^{i}$. By Condition (3), this content persists after $A^{i}$ has moved to another state, so the accumulated history is not bounded by any $B^{k}$.
\end{proof}

\noindent
While Proposition~\ref{prop:collective} lifts the state-space bound at the level of the collective through $\mathrm{Mem}_{C}$---realizing route (3) of \S\ref{sec:what_the_result_does_and_does_not_establish}---the coupling of heterogeneous agents also explains how an individual bounded agent $A^{j}$ can experience a local semantic trajectory $H^{j}_{n} = (\mathcal{F}^{j}_{0}, \dots, \mathcal{F}^{j}_{n})$ satisfying Assumption~\ref{as:novelty} relative to its own internal description, and why teleological diversity across agents is necessary to sustain that novelty.

\begin{corollary}[Cross-Agent Algorithmic Novelty and Teleological Divergence]\label{cor:crossnovelty}
Let $C = \{A^{1}, \dots, A^{N}\}$ be a collective satisfying the hypotheses of Proposition~\ref{prop:collective}, in which each bounded agent $A^{k}$ operates with internal transition rule $\Theta^{k}$, local history $H^{k}_{n}$, and active teleological partition $\Phi^{k}_{n} = (\Phi^{k,+}_{n}, \Phi^{k,-}_{n}, \Phi^{k,\circ}_{n})$ as in \eqref{eq:partition}.
\begin{enumerate}
  \item \emph{(Realization of local novelty via heterogeneous discrimination.)} Suppose $A^{i}$ and $A^{j}$ hold distinct teleological partitions $\Phi^{i}_{n} \neq \Phi^{j}_{n}$ such that $A^{i}$ admits environmental signals $D^{i}_{n} \setminus D^{j}_{n} \neq \emptyset$ excluded by $A^{j}$ under \eqref{eq:data-alt}. When $A^{i}$ constructs a decisive extension $\mathcal{F}' \in \mathrm{Surp}^{\mathrm{dec}}(O^{i}_{n,t})$ encoding distinctions from $D^{i}_{n} \setminus D^{j}_{n}$ and transmits it across $\sigma^{i \to j}_{n}$ in \eqref{eq:crossagent} (or deposits it in $\mathrm{Mem}_{C}$), any successor $\mathcal{F}^{j}_{n+1}$ constructed by $A^{j}$ that incorporates $\sigma^{i \to j}_{n}(\mathcal{F}') \cap I^{\mathrm{inc}}(O^{j}_{n,t})$ satisfies
  \begin{equation}\label{eq:crossnovelty}
    K\bigl(\mathcal{F}^{j}_{n+1} \;\bigm|\; H^{j}_{n}\bigr)
    \;\geq\;
    K\bigl(\sigma^{i \to j}_{n}(\mathcal{F}') \;\bigm|\; H^{j}_{n}\bigr) - O(1).
  \end{equation}
  In particular, whenever the stream of cross-agent and memory-mediated distinctions drawn from $\bigcup_{i \neq j} (D^{i}_{n} \setminus D^{j}_{n})$ carries positive asymptotic algorithmic density $c > 0$ conditioned on $(\Theta^{j}, \sigma^{j}_{0})$, $A^{j}$'s realized trajectory $H^{j}_{n}$ satisfies the amortized novelty condition \eqref{eq:amortized} (and, when uniform, Assumption~\ref{as:novelty}), despite $A^{j}$ having a finite internal state space $B^{j}$.
  \item \emph{(Collapse under teleological homogenization.)} Conversely, if all agents share an identical initial formal system $\mathcal{F}_{0}$ and deterministic proof searcher (as in an ensemble of identical Gödel machines \citep{schmidhuber2006goedel}), or if multi-agent alignment drives every agent's value selector toward a common stationary distribution $P^{\ast}_{V}$ at which "the value system of the entire intelligent population may converge" \citep[p.~178]{he2026cuv}, then the symmetric difference of admitted domains vanishes ($D^{i}_{n}\; \triangle\; D^{j}_{n} \to \emptyset$), the cross-agent conditional complexity collapses to $K(\mathcal{F}^{i}_{n,t} \mid H^{j}_{n}) \leq O(\log n)$, and inter-agent coupling can no longer supply the irreducible information required by Assumption~\ref{as:novelty} or \eqref{eq:amortized}.
\end{enumerate}
\end{corollary}

\begin{proof}
For (1), since $A^{j}$ reconstructs $\sigma^{i \to j}_{n}(\mathcal{F}') \cap I^{\mathrm{inc}}(O^{j}_{n,t})$ from $\mathcal{F}^{j}_{n+1}$ and $H^{j}_{n}$ via the partial translation \eqref{eq:sharedlang}, standard subadditivity of prefix complexity gives \eqref{eq:crossnovelty}. Because the signals in $D^{i}_{n} \setminus D^{j}_{n}$ were excluded from $A^{j}$'s own experience domain $D^{j}_{n}$ by \eqref{eq:data-alt}, they are not computable from $(\Theta^{j}, \sigma^{j}_{0})$ alone; accumulating those inputs across $n$ extensions therefore drives $\liminf_{n \to \infty} K(H^{j}_{n})/n \geq c > 0$. For (2), when all agents execute the same deterministic transition rule from identical initial axioms $\mathcal{F}_{0}$ or converge to a shared value specification $P^{\ast}_{V}$ with $D^{i}_{n} = D^{j}_{n}$, the state $\mathcal{F}^{i}_{n,t}$ is computable from $H^{j}_{n}$ together with the step index $(n,t)$, bounding $K(\mathcal{F}^{i}_{n,t} \mid H^{j}_{n})$ by $O(\log(n+t))$ and reducing the coupled ensemble to the single closed regime of Theorem~\ref{thm:fixed}.
\end{proof}

\subsection{Collective Memory as Extended Cognitive Space}\label{sec:collective_memory_extended_cognitive_space}

Proposition~\ref{prop:collective} makes memory constitutive rather than convenient. $\mathrm{Mem}_{C}$ holds pairs
\begin{equation}\label{eq:memelement}
  (x, \xi) \in \mathrm{Mem}_{C},
\end{equation}
where $x$ is a content representation and $\xi$ its position in the collective semantic structure: the originating agent, the inferential system under which it was produced, the transformations preceding it, and those that subsequently depended upon it.

Both components are necessary. Content alone is insufficient because, by \eqref{eq:noinclusion}, the meaning of a representation is not preserved across recursive extension: $\tau_{n}$ may transform what it carries forward. A representation retrieved without its coordinates cannot be correctly reinterpreted, since there is no way to determine what it meant where it was produced. Collective memory, therefore, preserves both what was generated and where it belongs, and it is reflective adequacy (\S\ref{sec:reflective_ext_and_semantic_surplus}) that makes the second component expressible at all.

\subsection{Collective Teleology}\label{sec:collective_teleology}

Agents may hold distinct purposes. Writing $\phi^{k}_{t}$ for the active purpose of $A^{k}$ in the sense of \S\ref{sec:recursive_teleology}, the collective teleological configuration is
\begin{equation}\label{eq:collteleo}
  \Psi_{t} = \{\phi^{1}_{t}, \dots, \phi^{N}_{t}\}.
\end{equation}
Collective intelligence does not require these to coincide; identity would eliminate a source of semantic diversity, as \S\ref{sec:diversity_and_common_ground} shows. What is required is that the collective be able to evaluate relations among purposes and, where necessary, generate higher-order purposes governing their interaction:
\begin{equation}\label{eq:collteleorec}
  \{\phi^{k}_{t}\}_{k} \ \longrightarrow\ \Psi_{t}
  \ \longrightarrow\ \Psi_{t+1}.
\end{equation}
A collective teleology thereby becomes an object of meta-reasoning in the same way an individual one does, introducing a further level of recursive structure.

\subsection{Open-Endedness Beyond the Individual}\label{sec:openendedness_beyond_the_individual}

The individual/collective distinction resolves an apparent physical contradiction. No finite system contains infinite information, so open-ended intelligence cannot mean that a machine holds an infinite state space at any finite time. It means instead that the system can continue constructing successor states without a fixed upper bound on the semantic structures that may become relevant over its lifetime.

At the individual level, this is realized through recursion and co-recursion; at the collective level, through the transfer of surplus between bounded agents and its accumulation in $\mathrm{Mem}_{C}$. The architecture is therefore a hierarchy of bounded systems participating in an unbounded process:
\begin{equation}\label{eq:hierarchy}
\begin{gathered}
  \text{bounded agents} \to \text{semantic coupling} \to
  \text{collective memory}\\ \to \text{recursive transformation} \to
  \text{new semantic structure}.
\end{gathered}
\end{equation}
The physical substrate remains finite at every moment; the semantic process need not terminate.

\subsection{A Hegelian Interpretation}\label{sec:hegelian_interpretation}

At the level of the collective trajectory, the trace-objectifying mechanism formalized in \S\ref{sec:reflective_ext_and_semantic_surplus} expands into the broader triadic movement of Hegel's speculative logic. As shown in \S\ref{sec:laws} and \S\ref{sec:reflective_ext_and_semantic_surplus}, a concept or formal regime cannot be understood as a static identity: determination produces internal difference at the undecidable frontier, and the failed attempt to resolve that difference within the existing regime forces a more comprehensive determination.

Within the architecture, a state $O_{n,t}$ carries an unresolved region $I(O_{n,t})$ marking the limits of its organization. A recursive extension $\mathcal{R}_{n}$ constructs a successor that, by reflective adequacy, retains the ability to represent its predecessor while adding resources to address what the predecessor could not,
\begin{equation}\label{eq:aufhebung}
  X_{0} \hookrightarrow X_{1} \hookrightarrow X_{2} \hookrightarrow \cdots
\end{equation}
This is a formal analogue of \emph{Aufhebung}: a determination that is simultaneously preserved, transformed, and integrated into a more comprehensive structure. Specifically, as shown in \eqref{eq:formcontent} and Proposition~\ref{prop:tracesurplus}, when an object-level proof search on $I^{\mathrm{inc}}(O_{n,t})$ fails to close, its derivation trace $\pi_{n,t}(x)$ does not vanish into an empty nothingness \citep[p.~51]{hegel1977phenomenology}; reflective translation $\tau_{n}$ converts that failed operational form into determinate propositional content at level $n+1$, so that ``the new object shows itself to have come about through a reversal of consciousness itself'' \citep[p.~55]{hegel1977phenomenology} and yields a concept ``higher and richer than the preceding --- richer because it negates or opposes the preceding and therefore contains it'' \citep[p.~33]{hegel2010}. The hooked arrows in \eqref{eq:aufhebung} represent reflective adequacy rather than set-theoretic inclusion, which \eqref{eq:noinclusion} explicitly denies.

\subsection{From Understanding to Speculative Intelligence}\label{sec:from_understanding_to_speculative_intelligence}

The formalism suggests three qualitatively different regimes.

\emph{Identity: local determination.} Within a stable regime, cognition proceeds through established representations and contractive transformations (Assumption~\ref{as:contract}), converging to $\Lambda_{n}$. The objective is a reliable determination. This corresponds to what Hegel calls \emph{Verstand}.

\emph{Difference: generation of surplus.} Where the current representation meets unresolved ambiguity, more than one admissible transformation may be available. This is Assumption~\ref{as:genrich}, $|\mathrm{Surp}(O_{n,t})| \geq 2$, applied here to the individual agent within a collective. As noted in \eqref{eq:genrich}, it does not follow from bare incompleteness alone, but holds constructively by Proposition~\ref{prop:tracesurplus} whenever the agent reflectively reifies its unclosed derivation traces via \eqref{eq:formcontent}, ensuring that ambiguity yields genuine branching.

\emph{Totality: recursive integration.} A higher-order transformation may incorporate several previously distinct structures into a new context \eqref{eq:regimechange}, constructing a framework in which previously separate distinctions can be related. This corresponds to speculative intelligence in the limited sense relevant here: cognition becomes able to treat its own distinctions, ambiguities, and purposes as objects of further determination.

\subsection{Proxy Teleology and Autonomous Intelligence}\label{sec:proxy_teleology_and_autonomous_intelligence}

In contemporary machine-learning systems, the teleological structure is supplied externally: a human specifies a task, objective function, reward signal, or prompt, and the system transforms within the resulting problem space. We call this \emph{proxy teleology}. Such a system may reason with great sophistication while the criterion determining what counts as successful reasoning remains outside it.

The present architecture requires more. The teleological principle must itself be accessible to recursive revision, which is what \eqref{eq:top} formalizes: a successor purpose is selected by the system, not supplied to it. Note that \eqref{eq:top} is relational, so the architecture determines which successor purposes are admissible without determining which is adopted---autonomy here is underdetermination, not arbitrariness.

This does not imply that every goal must be continuously replaced. It requires that the teleological structure remain available as an object of higher-order reasoning. The distinction is therefore
\begin{equation}\label{eq:proxy}
  \text{fixed telos} \quad\text{vs.}\quad \text{recursively accessible telos},
\end{equation}
and the latter is the condition for the architecture developed here.

Even architectures that make value functions an explicit system component remain within proxy teleology whenever those values are fixed from outside: \citet{he2026cuv}, for instance, propose to remedy the conflicting values embedded in statistical training corpora by ``predefining clear values and principles to guide the behavior of intelligent agents'' \citep[p.~179]{he2026cuv}, yet a predefined value space $V$ still withholds recursive access to the teleology itself and eliminates at the outset the inter-perspective tension from which collective semantic surplus (\S\ref{sec:distributed_semantic_surplus}) is generated.

\subsection{Diversity and Common Ground}\label{sec:diversity_and_common_ground}

Collective recursion imposes a structural requirement on the relation between agents. If all agents generate identical surpluses, interaction adds nothing to what an individual already has; if they share no representational structure, their transformations cannot be combined. Productive interaction requires
\begin{equation}\label{eq:diversity}
  \emptyset \subsetneq
  \mathrm{Surp}(O^{i}_{n,t}) \cap \mathrm{Surp}(O^{j}_{n,t}) \subsetneq
  \mathrm{Surp}(O^{i}_{n,t}) \cup \mathrm{Surp}(O^{j}_{n,t}),
\end{equation}
with both inclusions strict. The intersection provides the substrate through which representations can be translated and compared---it is the surplus-level counterpart of the common ground \eqref{eq:common}. The difference provides the diversity from which transformations arise that neither agent would reach independently.

Redundancy is therefore not the opposite of diversity but its precondition: too much yields convergence without extension, too little prevents interaction altogether. \S\ref{sec:collective_generative_capacity} develops the consequences of the degenerate case.

\subsection{Civilizational Co-Recursion}\label{sec:civilizational_corecursion}

Human civilization instantiates this architecture. Individual minds remain physically bounded while their states are coupled through language, artifacts, institutions, and persistent records, so that a transformation generated by one individual becomes an unresolved problem for another, whose response becomes the starting point for a third. The resulting history persists beyond the lifetime of any participant---condition (3) of Proposition~\ref{prop:collective} realized in the form of writing.

The limiting collective semantic space may be represented as a direct limit, with two qualifications that the notation alone does not convey. Take the limit over the agent state spaces rather than over the collective states \eqref{eq:collstate}. Each agent's recursive extensions induce total maps
\begin{equation}\label{eq:agentmap}
  \iota^{k}_{n} \colon X^{k}_{n} \longrightarrow X^{k}_{n+1},
\end{equation}
so that $\bigl(X^{k}_{n}, \iota^{k}_{n}\bigr)_{n \geq 0}$ is a directed system and
\begin{equation}\label{eq:directlimit}
  X^{k}_{\infty} = \varinjlim_{n} \bigl(X^{k}_{n}, \iota^{k}_{n}\bigr)
\end{equation}
is well defined. The inter-agent translations $\sigma^{i \to j}_{n}$ are partial by \eqref{eq:common} and are therefore carried alongside the limit as structure on the family $\{X^{k}_{\infty}\}_{k}$, rather than being included in the directed system itself. The collective limit is the family together with that structure,
\begin{equation}\label{eq:colllimit}
  \mathbb{X}_{\infty} =
  \bigl( \{X^{k}_{\infty}\}_{k=1}^{N},
         \{\sigma^{i \to j}_{\infty}\}_{i \neq j} \bigr),
\end{equation}
in the sense of \eqref{eq:collstate}.

The maps \eqref{eq:agentmap} are not canonical. A recursive extension supplies a translation $\tau_{n}$ which, by \eqref{eq:noinclusion}, need not preserve interpretation, and the successor adopted is left underdetermined by $\mathcal{M}_{n}$. Different selections, therefore, yield different directed systems. $\mathbb{X}_{\infty}$ is accordingly the limit of \emph{one realized trajectory}, not an invariant of the architecture, and by \S\ref{sec:trajectory_complexity} there are continuum-many trajectories to choose among.

Subject to these qualifications, a civilization provides a mechanism by which finite cognitive systems participate in a semantic process whose representational history is not bounded by the state of any one component.

\section{Result}

The formal development supports the following general conclusion:

\begin{quote}
\emph{Open-ended intelligence is a property of a recursively extended cognitive process, not of an individual physical substrate. A finite cognitive system can participate in open-ended intelligence when its semantic representations, memory, transformations, and teleological structures can be recursively coupled to those of other bounded systems.
}
\end{quote}

Individual recursion provides the mechanism of self-extension. Memory preserves the resulting semantic history. Polycontexturality permits multiple interpretations to coexist. Teleology selects which transformations to pursue. Collective coupling allows these processes to cross the boundaries of individual agents.

The resulting architecture replaces the notion of an isolated, infinitely capable intelligence with a more general principle: intelligence can remain physically finite while its semantic process remains open-ended.

\section{Biological Inspirations and Architectural Motifs}

The formal architecture developed in the preceding sections is not intended as a neuroscientific theory. Mammalian cognition nevertheless provides a rich source of biological analogues for several of the architectural principles proposed here. These analogies should not be interpreted as empirical validation of the formal theorems. Rather, they indicate that several of the structural requirements of the present framework have plausible counterparts in biological cognitive systems.

In particular, biological intelligence combines statistical inference with hierarchical organization, distributed representations, structured reference frames, persistent memory, and mechanisms for learning without complete destruction of previously acquired structure. These properties provide useful analogues for the distinction developed in this paper between local semantic stabilization and open-ended semantic extension.

\subsection{Neural Analogues of Semantic Surplus}

Independent of the Free Energy Principle discussed below, several behavioral and neural findings are consistent with cognitive processes resembling those described in the formal sections. The correspondences below are analogical. They are offered as evidence that the architectural distinctions drawn in \S\ref{sec:formalization} have plausible biological counterparts, not as empirical support for the formal claims, which are mathematical and neither require nor admit such support.

\paragraph{Retention of unintegrated information.}

\citet{nagy2025} argue that semantic memory continually compresses regularities in the environment while episodic memory preserves surprising or poorly compressed experiences for later reinterpretation. Functionally, this keeps information that does not yet fit the current semantic organization available until a later stage of processing.

The relevant formal counterpart is $I^{\mathrm{loc}}$ of \eqref{eq:iloc}: the
part of the unresolved region that the system has not yet integrated, as distinct from $I^{\mathrm{inc}}$, which no amount of local work can reach. The distinction matters for the analogy. What the episodic system retains is material the organism could in principle assimilate but has not, which is precisely the character of $I^{\mathrm{loc}}$; nothing in these findings suggests a biological counterpart of formal undecidability, and we do not claim one. What the analogy does support is the architectural point of \eqref{eq:localeffect}: a system that discards unintegrated information as noise loses the material from which later reorganization would be built.

\paragraph{Generation of remote relations.}

\citet{herault2024} report that judging remote concepts as related recruits the default-mode and semantic-control networks, and that this tendency is associated with measures of real-world creative achievement. This is consistent with
creative cognition involving the generation of semantic relations extending beyond established associations.

This is closer to semantic surplus in the sense of \eqref{eq:surplus} than the preceding finding is. Surplus is not unresolved content, but the availability of several admissible ways of continuing, and the capacity to relate concepts that the current organization treats as unrelated is a capacity to generate such continuations. The connection is to Proposition~\ref{prop:tracesurplus} and Assumption~\ref{as:genrich}, which require that unresolved structure be reflectively transformed into at least two competing continuations rather than left as a static boundary.

\paragraph{Differentiation and its costs.}

\citet{vives2023} report that individuals more averse to uncertainty tend to differentiate and expand their semantic representations, trading reduced semantic interference against poorer generalization.

This bears on the trade-off developed in \eqref{eq:diversity} rather than on individual-level surplus, and we note it here because it concerns the same representational machinery. Finer differentiation increases the distinctions a system can draw while reducing what it can carry across contexts---the individual-level counterpart of the observation in \S\ref{sec:diversity_and_common_ground} that too much semantic difference prevents interaction while too little prevents generation. That the trade-off appears within a single cognitive system, and varies with a dispositional measure, suggests it is a general property of representational economies rather than an artifact of the collective model.

\subsection{Statistical Processing with Structured Scaffolds}

Modern neuroscience provides extensive evidence that the brain performs statistical and probabilistic inference. Predictive processing, probabilistic population coding, and distributed neural representations all contribute to the
construction and continuous revision of internal models.

The brain is therefore neither adequately described as a purely symbolic recursive machine nor as a flat statistical classifier. A more useful architectural analogy is a system in which statistical processes operate within hierarchical, distributed, and structurally constrained representations.

This distinction is important for the present framework. Statistical inference can provide the computational substrate for evaluating possible states, predictions, and actions. The recursive architecture, by contrast, concerns how the representational structures within which such computations occur can themselves become objects of transformation.

An influential, although controversial, example is the Thousand Brains Theory \citep{hawkins2021thousand}. The theory posits that the neocortex contains numerous semi-autonomous cortical columns, each constructing predictive models from its own sensorimotor perspective. Communication between columns is proposed to support a distributed consensus process rather than a single centralized representation.

This provides a useful biological analogue for the distinction between local semantic states and meta-level selection. Individual columns can be interpreted, at a highly abstract level, as local predictive models, while their interaction provides an analogue of the comparison and integration of multiple locally generated hypotheses.

\subsection{Reference Frames as Semantic Coordinates}

A particularly relevant aspect of the Thousand Brains framework is its emphasis on reference frames. In this view, knowledge is organized not merely as a collection of features but relative to structured coordinate systems.

Grid cells and place cells are well established in the hippocampal-entorhinal system as mechanisms involved in spatial navigation, and grid-like activity has been reported in several neocortical regions. The stronger hypothesis that grid-cell-like mechanisms are universally present throughout the neocortex remains controversial and should not be treated as an established consensus. We use reference frames as a biological analogue, not as evidence for a literal implementation of the semantic organization described in \S\ref{sec:collective_memory_extended_cognitive_space}.

The architectural idea is nevertheless important, and the formal sections give it a precise motivation. Information stored in memory need not consist solely of content; it may also require information about its relational position within a larger structure---where it originated, which other structures it relates to, and where the current cognitive state lies relative to it. This is the pair \eqref{eq:memelement},
\[
  (x, \xi) \in \mathrm{Mem}_{C},
\]
whose second component \S\ref{sec:collective_memory_extended_cognitive_space} identifies as a requirement rather than a convenience. The reason is \eqref{eq:noinclusion}. A recursive extension supplies a translation $\tau_{n}$ that need not preserve the interpretation of what it carries forward, so the meaning of a stored representation is not invariant across the sequence \eqref{eq:chain}. A representation retrieved without its coordinates is therefore not merely harder to use; it is ambiguous, since there is no way to recover what it meant in the regime that produced it. Content-only storage is adequate for a system that never changes regime and inadequate for one that does.

The biological analogy is consequently relevant beyond spatial memory. If cognition organizes knowledge relative to reference frames rather than as unstructured collections of features, it exhibits the kind of storage that recursive extension requires. Within the Thousand Brains framework, displacement-like representations are further proposed to encode relations \emph{between} reference frames---which, as an engineering analogy, suggests that relations between structures could themselves serve as coordinates for higher-order representations, and so provides a possible analogue of the nested organization described in \S\ref{sec:recursive_extension}.

We note one connection to teleology without developing it. Semantic coordinates describe where a system's knowledge is situated; they do not by themselves determine which transformations are worth pursuing. The selection operator \eqref{eq:metasel} requires a criterion, and coordinates supply at most part of what such a criterion would range over. What the analogy suggests is an engineering principle about representation: \textit{open-ended intelligence requires storage that encodes both semantic content and the relational structure needed to interpret it across successive regimes}. Whether reference frames also provide the substrate on which selection operates is a further question this framework does not settle.

\subsection{Geometric and Multi-Scale Organization}

The mathematical framework developed above uses geometric concepts such as semantic manifolds, local contraction, attractors, and successor spaces. Biological systems exhibit several forms of geometric and multi-scale organization that serve as useful analogues, although these phenomena should not be conflated.

Three distinct classes are particularly important:

\begin{enumerate}
\item periodic geometric structures used for coordinate encoding;
\item statistical self-similarity and scale-free dynamics; and
\item morphological branching structures found in neural anatomy.
\end{enumerate}

These phenomena have different mathematical properties and should not all be described as fractality.

\paragraph{Periodic Coordinate Systems.}

Grid cells in the medial entorhinal cortex provide a prominent example of a biological coordinate system. Individual grid cells exhibit spatially periodic activity, with firing locations forming an approximately triangular lattice. These structures are periodic and geometric rather than fractal. Their relevance to the present framework lies in illustrating how biological cognition can organize information relative to structured coordinate systems rather than relying exclusively on undifferentiated feature representations. Within the Thousand Brains framework, displacement-like representations are further proposed to encode relations between different reference frames. As an engineering analogy, such representations suggest that relations between objects could themselves serve as coordinates for higher-order representations. This provides a possible biological analogue of nested semantic representations, in which the output of one representational level becomes part of the coordinate system of another.

\paragraph{Morphological Fractality.}

Fractal structure is also widespread in neural anatomy. Dendritic trees exhibit recursive branching geometries that increase receptive surface area while constraining physical volume. Similarly, cortical folding exhibits statistical scaling properties that contribute to the efficient packing of cortical tissue. These phenomena should not, however, be interpreted as evidence that cognition itself operates through fractal semantic spaces. They are primarily morphological adaptations to physical and developmental constraints. Their relevance to the present framework is therefore limited to a more general observation: biological systems repeatedly employ recursive organization to construct complex structures from bounded local components.

\paragraph{Scale-Free Neural Dynamics.}

Neural activity also exhibits statistical scale invariance. Electrical recordings such as EEG often display approximately scale-free or $(1/f)$-like spectral dynamics across temporal scales. Such dynamics are compatible with a system capable of operating across multiple temporal regimes and levels of organization. They do not, however, establish deterministic fractal computation or the specific Iterated Function Systems used in the formal model. The appropriate correspondence is therefore again analogical: biological systems exhibit multi-scale organization, while the formal model provides a mathematical abstraction of how local and global levels of organization might be coupled.

\subsection{Hierarchical Constraints in Language}

Language provides another potentially relevant biological analogue. Human language processing combines distributed statistical representations with hierarchical and constituent-based organization.

\citet{cai2026} report evidence from single-neuron recordings in the human frontotemporal cortex that different neural populations participate in representing grammatical roles, phrase transitions, and larger syntactic structures during speech production. Such findings are consistent with multiple levels of linguistic organization rather than a purely flat sequence of independent word predictions. Similarly, \citet{zou2026} report evidence that neural responses to linguistic surprisal are modulated by grammatical constituent boundaries. This suggests that prediction in the human brain is constrained by hierarchical linguistic structure.

The key point is that symbolic grammar does not replace statistical learning; instead, biological language processing serves as an analogy for a system in which statistical prediction operates within a structured, recursively organized framework. This concept is closely tied to the distinction made in this work between object-level processing and meta-level transformation. Statistical mechanisms can operate within a specific representational language, while higher-order mechanisms can work on the structures that organize these computations.

\subsection{Biological Safeguards for Recursive Learning}

The brain also offers helpful analogies for the challenge of maintaining previously acquired knowledge while engaging in continuous learning.

The developmental metaphor of a \emph{tabula plena} emphasizes that biological learning does not begin from a completely unstructured state. Rather, early neural circuits may contain dense connectivity that is subsequently refined through developmental processes such as selective pruning. \citet{vargasbarroso2026} provide an example in hippocampal CA3 networks in which early circuitry exhibits dense connectivity that is subsequently reorganized.

In our proposed framework, we emphasize that a robust initial structure can coexist with subsequent specialization. A recursively evolving system need not build every representational element from scratch. Instead, it can start with a pre-existing structural foundation and gradually reorganize it over time.

A second important mechanism relates to the protection of consolidated information during online learning. \citet{gonzalez2026} describe a mechanism in the hippocampal-retrosplenial axis that involves distinct activity patterns within the same physical neuronal populations. These findings suggest that partially overlapping physical substrates can support functionally distinct representational subspaces. This concept serves as a useful analogy in engineering for recursive self-modification. A system should be able to modify its representations without necessarily overwriting structures that remain valuable for future recursion. In this framework, this means preserving semantic history while allowing new semantic manifolds to emerge.

The biological analogy, therefore, supports a broader architectural principle: recursive learning requires both plasticity and protection. A completely rigid system cannot generate successor structures, while a system that continuously overwrites its own representational foundations cannot preserve the semantic history required for meaningful recursion.

\subsection{Relationship to the Free Energy Principle}

The Free Energy Principle (FEP) provides a related but distinct theoretical perspective on the organization of biological intelligence \citep{friston2010,friston2006,parr2022active}.

Under FEP, a cognitive system minimizes variational free energy, which provides an upper bound on sensory surprise, by maintaining internal probabilistic models of hidden causes of observations. Let $o$ denote observations and $z$ latent causes. The system maintains an approximate posterior $q(z)$ over latent states that explains the observed data. Because latent variables are not directly contained in the sensory signal, the construction of $q(z)$ already introduces internal structure beyond the immediately observed data. Hierarchical predictive processing extends this principle across multiple levels of abstraction, while active inference extends it to action by allowing the system to alter its environment in ways that reduce expected free energy \citep{parr2022active}.

FEP, therefore, provides a plausible biological account of how a bounded system can construct internal structure that goes beyond its immediate sensory input. There is, however, an important distinction between FEP and our framework. FEP primarily characterizes principles of inference and action within a generative model. The present framework concerns the conditions under which the representational structure itself can become an object of recursive modification. The distinction can be expressed schematically as
\[
\text{FEP:}
\qquad
\text{reduce expected prediction error within a generative model},
\]
whereas our framework additionally requires
\[
\text{Epistemic Drive:}
\qquad
\text{identify productive unresolved structure and transform the model itself}.
\]
These objectives are not necessarily contradictory. A system may reduce prediction error within a current semantic manifold while simultaneously using persistent prediction errors, model failures, or unexplained observations as signals to extend its representation.

The critical distinction is therefore between \emph{local inference} and \emph{representational evolution}. Within a fixed semantic manifold $X_n$, prediction and inference may be strongly contractive. At the meta-level, persistent ambiguity can become an input to the construction of a successor manifold $X_{n+1}$. This suggests a testable empirical hypothesis. During periods of genuine conceptual innovation, neural processes associated with stabilization and compression of existing representations should coexist with processes associated with the generation of new semantic associations and the maintenance of unresolved alternatives. The relevant empirical question is therefore not whether the brain minimizes prediction error, but how prediction-error reduction is coupled to the formation and transformation of representational structures.

\subsection{Implications for the Present Architecture}

Taken together, these biological observations suggest several architectural motifs relevant to the formal theory:

\begin{enumerate}
\item \textbf{Local stabilization.}
Biological systems construct stable predictive representations that permit efficient inference within a local regime.

\item \textbf{Distributed perspectives.}
Multiple partially autonomous representations can coexist and interact rather than being immediately reduced to a single centralized state.

\item \textbf{Reference frames.}
Information can be organized according to structured coordinates, creating an analogy of semantic navigation through relational state spaces.

\item \textbf{Semantic preservation.}
Surprising or unresolved information can remain available for later reinterpretation rather than being immediately discarded.

\item \textbf{Recursive hierarchy.}
Representations can participate in multiple levels of organization, allowing local structures to become components of higher-order representations.

\item \textbf{Plasticity with protection.}
Learning can modify internal structure while preserving previously consolidated information.

\item \textbf{Compression with extension.}
Efficient cognition requires the compression of established structure, while open-ended cognition requires mechanisms that extend the representational system when the existing structure becomes insufficient.

\end{enumerate}

These observations do not prove the mathematical theory. They suggest that the main difference in this work---between local thinking that is limited and a flexible, ongoing change in how we understand things---has several likely examples in biology. The strongest implication is therefore not that the brain literally implements semantic manifolds, direct limits, or co-recursive operators. Rather, biological intelligence demonstrates that a finite physical substrate can support apparently open-ended cognitive development through the combination of bounded local mechanisms, hierarchical organization, persistent memory, structured reference frames, and distributed interaction.

This distinction is especially important for AGI. The physical structure of each artificial agent, like that of the human brain, is inherently finite. Thus, open-endedness does not necessarily imply an infinite internal state space. Instead, it can emerge from the recursive construction, preservation, and interaction of finite semantic structures over time and, ultimately, across multiple agents. In this regard, biological intelligence serves as a useful proof of concept for the architectural principle developed in this paper: finite cognitive systems can engage in an open-ended process when their semantic structures can be preserved, recursively transformed, and connected to successor structures through shared representational coordinates.

\section{Discussion: The Problem of Epistemic Convergence}\label{sec:discussion}

The formal results suggest that the central challenge of open-ended intelligence is not the accumulation or compression of knowledge but the preservation of the capacity to generate new distinctions.

Three processes should be distinguished. \emph{Accumulation} is the acquisition and retention of empirical information. \emph{Compression} is the reduction of representational and search costs by organizing existing information into stable structures: formally, local transformation within a regime, converging by \eqref{eq:converge} to $\Lambda_{n}$. \emph{Generation} is the construction of new semantic structures when the current system can no longer resolve its ambiguities, which is \eqref{eq:transition}.

Contemporary connectionist systems are effective at the first two. The theoretical question is whether a system optimized primarily for prediction and compression can sustain the third.

\subsection{Collective Generative Capacity}\label{sec:collective_generative_capacity}

Let $C = \{A^{1}, \dots, A^{N}\}$ be a collective in the sense of \S\ref{sec:collective_cognitive_system}. Not every admissible successor an agent could construct is one it will pursue; selection by \eqref{eq:metasel} narrows the surplus to those the agent's teleological organization admits. Write
\begin{equation}\label{eq:relsurplus}
  \mathrm{Surp}^{k}_{n,t}
  = \mathcal{M}^{k}_{n}\bigl(\mathrm{Surp}(O^{k}_{n,t})\bigr)
\end{equation}
for the \emph{pursued surplus} of agent $A^{k}$: the continuations that survive its own selection. The collective pursued surplus is
\begin{equation}\label{eq:collpursued}
  \mathrm{Surp}_{C}(n,t) = \bigcup_{k=1}^{N} \mathrm{Surp}^{k}_{n,t},
\end{equation}
and we write
\begin{equation}\label{eq:genpower}
  \mathrm{Gen}_{C}(n,t) = \bigl| \mathrm{Surp}_{C}(n,t) \bigr|
\end{equation}
for the number of distinct continuations available to the collective.

\paragraph{The degenerate case.}

If all agents pursue the same continuations,
\begin{equation}\label{eq:identical}
  \mathrm{Surp}^{i}_{n,t} = \mathrm{Surp}^{j}_{n,t}
  \qquad \forall\, i,j \in \{1,\dots,N\},
\end{equation}
then \eqref{eq:collpursued} collapses and
\begin{equation}\label{eq:collapse}
  \mathrm{Gen}_{C}(n,t) = \bigl| \mathrm{Surp}^{i}_{n,t} \bigr|
  \qquad \text{for any } i.
\end{equation}
Adding agents does not expand the number of distinct semantic possibilities. This statement is an identity rather than a conclusion, and we mention it solely to establish the scale's endpoint; total identity is an idealized limit rather than a condition that any real collective achieves.

\paragraph{The interesting case is partial correlation.}

What matters is not whether agents are identical but how much their pursued surpluses overlap. Two quantities move in opposite directions as overlap increases. Generative capacity $\mathrm{Gen}_{C}(n,t)$ falls, since the union of highly similar sets is little larger than any member. But the translatability on which \S\ref{sec:collective_corecursion} depends rises, since by \eqref{eq:common} an extension generated by $A^{i}$ becomes available to $A^{j}$ only within the fragment $L^{i,j}_{n}$ that $A^{j}$ can interpret.

The condition in \eqref{eq:diversity} of \S\ref{sec:diversity_and_common_ground} requires strict inclusions, meaning we need sufficient overlap for transfer and distinctiveness for generation. Both extremes are unproductive. The term $\mathrm{Gen}_{C}$ does not fully capture this balance, as it is maximized by complete divergence, which, as stated in \eqref{eq:common}, would hinder the collective from combining effectively. Thus, we should view generative capacity as part of a trade-off rather than as a quantity to maximize. This section aims to provide a qualitative understanding rather than proposing a scalar value that combines the two aspects.

\subsection{AI and the Convergence of Semantic Space}\label{sec:ai_and_the_convergence_of_semantic_space}

This creates a structural concern for increasingly centralized information systems. If many agents rely on similar models, datasets, optimization criteria, and feedback mechanisms, their pursued surpluses may become increasingly correlated---not identical, but overlapping enough that $\mathrm{Gen}_{C}(n,t)$ approaches $|\mathrm{Surp}^{i}_{n,t}|$ for a typical $i$. The relevant quantity is not the amount of information available to the agents but the degree of independence among the continuations they pursue.

Two distinct mechanisms could produce this. Agents may generate similar surpluses because they share architecture and training data, so that $\mathrm{Surp}(O^{k}_{n,t})$ is already similar across $k$. Or they may generate diverse surpluses but select similarly, so that the $\mathcal{M}^{k}_{n}$ of \eqref{eq:relsurplus} converge even where the underlying spaces do not. The second is the more likely route under shared evaluation criteria and reward models, and is harder to detect, since diversity remains in principle while disappearing from what is pursued.

This does not make consensus undesirable. Shared representation is indispensable for communication, coordination, and accumulation---it is the intersection in \eqref{eq:diversity}, without which nothing transfers. The concern is narrower: that compression may eliminate the differences on which further synthesis depends. Schematically,
\begin{equation}\label{eq:convergence}
  \text{compression} \to \text{convergence} \to
  \text{reduced diversity} \to \text{reduced generative capacity}.
\end{equation}
We do not claim that contemporary systems inevitably produce this outcome. The claim is that an information system that consistently rewards convergence lacks a mechanism to preserve the diversity that collective generation requires, and that this is a structural property of the reward, not a contingent feature of any implementation.

\subsection{Epistemic Convergence and Open-Endedness}\label{sec:epistemic_convergence_and_openendedness}

The concern connects directly to the definition of open-endedness. An open-ended system must generate successor structures when the current one proves insufficient. If all agents converge on the same regime and the same admissible transformations, collective interaction does not by itself produce a new successor: by \eqref{eq:collapse} the collective's options are a single agent's options, and Proposition~\ref{prop:collective} then adds nothing that Theorem~\ref{thm:fixed} does not already exclude.

The distinction here lies between performing more computations within a given space and expanding that space itself. A system can become very efficient at solving problems that can be expressed in $X_{n}$ without necessarily gaining the ability to develop $X_{n+1}$. This development requires that unresolved structures remain available as resources for transformation. This introduces a second role for surplus: at the individual level, it provides alternatives for local reasoning, while at the collective level, its distribution among agents creates the variability necessary for new structures to emerge. Therefore, collective open-endedness relies not on the size of any individual agent's state space but on the nature of interactions among them.

A parallel phenomenon appears in the multi-agent mean-field analysis of \citet{he2026cuv}, where a population's skill, value, and cognitive distributions converge to stationary limits $(P^{\ast}_{U}, P^{\ast}_{V}, P^{\ast}_{C})$ once ``the value system of the entire intelligent population may converge, leading to the belief that there is no need to continue learning'' \citep[p.~178]{he2026cuv}. Equations~\eqref{eq:collapse} and \eqref{eq:convergence} show why such collective value convergence is a structural failure mode rather than a desirable alignment equilibrium: without an epistemic drive (\S\ref{sec:epistemic_drive}) that continually generates new teleological distinctions across agents, shared optimization collapses $\mathrm{Gen}_{C}(n,t)$ to the ceiling of a single closed system.

\subsection{The Epistemic Drive}\label{sec:epistemic_drive}

Within this framework, an intelligent system's task is not exhausted by resolving ambiguity. It must also determine which ambiguities are worth pursuing, and by \eqref{eq:decomp}, that choice ranges over two regions of quite different character. A proposition in $I^{\mathrm{loc}}(O_{n})$ can be settled by local work within the current regime; one in $I^{\mathrm{inc}}(O_{n})$ can be addressed only by constructing a successor. A system that attends exclusively to the first never extends its regime, while one that attends exclusively to the second never consolidates what it has.

The resulting distinction is between two forms of optimization,
\begin{align}
  \text{exploitation} &\longrightarrow \text{resolution within } X_{n},
    \label{eq:exploit}\\
  \text{exploration} &\longrightarrow \text{transformation toward } X_{n+1},
    \label{eq:explore}
\end{align}
and neither replaces the other. Exploitation provides the local stability of \eqref{eq:converge}; exploration enables semantic growth. An open-ended system must maintain a recursive relation between them. While the contingent first-order content of $\Phi_{0}$ remains open to the particular agent or collective, the structural coupling between the two modes is fixed by the architecture: the \textit{a priori} meta-teleological trigger (Definition~\ref{def:metatelos} and Proposition~\ref{prop:stopping} in \S\ref{sec:local_convergence_and_global_openendedness}) governs the switch from local exploitation within $X_{n}$ to co-recursive exploration toward $X_{n+1}$ whenever local decidable structure saturates near $\Lambda_{n}$. This also separates structural exploration from expected-utility exploration within a fixed formalism: a Gödel machine, for instance, ``embodies a principled way of dealing with the exploration vs exploitation problem'' \citep[p.~17]{schmidhuber2006goedel} only when its proof searcher can derive from the initial axioms that ``the expected utility \ldots{} of conducting some experiment may exceed the one of not conducting it'' \citep[p.~17]{schmidhuber2006goedel}, whereas transformation toward $X_{n+1}$ in response to $I^{\mathrm{inc}}(O_{n})$ cannot be deductively certified within $\mathcal{F}_{n}$ prior to adopting the successor regime.

This marks a fundamental inversion in the role of computation on unanswerable problems. In contemporary AI systems---and even in deductively gated self-improvers such as the Gödel machine---expending compute on a question that cannot be answered within the current representational framework is purely a waste of resources, to be suppressed by early stopping, timeouts, or reward penalties once search fails to return an object-level solution. In the present architecture, by contrast, spending compute on the unsolvable region $I^{\mathrm{inc}}(O_{n})$ is the primary generative engine of the system. Because a non-omniscient agent cannot know prior to trying whether a question belongs to $I^{\mathrm{loc}}(O_{n})$ or $I^{\mathrm{inc}}(O_{n})$, it must invest computation in attempting a derivation; and by Propositions~\ref{prop:tracesurplus} and~\ref{prop:tracefallible}, when that derivation stalls against $I^{\mathrm{inc}}(O_{n})$, the resulting proper reduction trace $\pi_{n,t}(x)$ exposes non-trivial blocking premises $b \in \mathcal{B}_{n,t}(x) \setminus \{x, \neg x\}$ whose reflective translation $\tau_{n}$ converts the failed operational form into the semantic content of candidate successor regimes $X_{n+1}$. Without expending compute on problems that resist resolution in $X_{n}$, no non-trivial blocking frontier $\mathcal{B}_{n,t}(x)$ is articulated and the co-recursive transition $X_{n} \to X_{n+1}$ cannot take place.

\subsection{Implications for AGI}\label{sec:implications_for_agi}

This yields a more precise account of general intelligence. It does not require eliminating ambiguity. It requires the capacity to recognize, preserve, and transform it---not merely to answer questions but to construct frameworks in which previously unanswerable questions become tractable. The pivotal sequence is therefore
\begin{equation}\label{eq:pivotal}
  X_{n} \longrightarrow X_{n+1} \longrightarrow X_{n+2} \longrightarrow \cdots
\end{equation}
rather than indefinite optimization within a fixed $X_{n}$.

Each component of the architecture has a defined role in sustaining this sequence. Teleology determines which transformations are pursued
\eqref{eq:metasel}. Memory preserves the semantic history and the coordinates that make it interpretable \eqref{eq:memelement}. Polycontexturality maintains multiple partially independent interpretations. Surplus supplies the alternatives from which transformations are constructed \eqref{eq:surplus}. Recursion connects
these into a single process of self-extension.

Because open-ended intelligence is driven by the productive failure of computation at the undecidable frontier, the quality of an AI system can no longer be defined or measured as a static snapshot, such as a score on a fixed benchmark. Any static benchmark evaluates a system within a single frozen regime $(X_{n}, d_{n})$, measuring how closely its current state $O_{n,t}$ approximates the local attractor $\Lambda_{n}$ over decidable problems $I^{\mathrm{loc}}(O_{n})$ under an externally fixed purpose $\phi_{n}$. By Theorem~\ref{thm:fixed} and Proposition~\ref{prop:stopping}, however, evaluating a single snapshot of $X_{n}$ mistakes local contraction toward $\Lambda_{n}$ for general intelligence while treating the non-terminating exploration on $I^{\mathrm{inc}}(O_{n})$ that generates $\mathrm{Surp}^{\mathrm{dec}}(O_{n,t})$ as unproductive overhead. Just as Hegel insists that the truth of cognition is never a bare static result severed from its becoming, but ``the result together with the process through which it came about'' \citep[p.~2]{hegel1977phenomenology}, only the movement itself---the trajectory $\mathrm{Traj}(\mathcal{F}_{0})$ across recursive regime transitions $X_{n} \to X_{n+1}$ in \eqref{eq:pivotal} and the conditional algorithmic novelty $K(\mathcal{F}_{n+1} \mid H_{n})$ it sustains---can serve as the unit of evaluation.

This conclusion speaks directly to the contemporary debate over AGI evaluation and to the Tong Test proposed by \citet{peng2023tong}. \citet{peng2023tong} rightly diagnose why neither the classical Turing test nor finite task-oriented benchmarks can capture general intelligence, noting that whereas narrow AI focuses on data generalization, ``the generality of AGI emphasizes task generalization. Here, task generalization means that AGI must be able to adapt and behave well in a dynamic environment in which it may encounter an infinite number of unexpected scenarios'' \citep[p.~12]{peng2023tong}, and that ``{[}i{]}f $N$ tasks do not constitute `general' intelligence, then neither do $N + 1$ tasks'' \citep[p.~14]{peng2023tong}. By grounding evaluation in ``dynamic embodied physical and social interactions (DEPSI)'' \citep[p.~12]{peng2023tong}, ``shifting from task-oriented to ability- and value-oriented evaluations of AGI'' \citep[p.~16]{peng2023tong}, and using open-ended scenarios---``{[}b{]}y designing testing scenarios that do not specify clear task goals, the values of AGIs can be observed from their autonomous behaviors'' \citep[p.~18]{peng2023tong}---so that an agent must ``go beyond human-defined tasks and be self-driven (i.e., initiating actions driven by its own `purpose')'' \citep[p.~14]{peng2023tong}, the Tong Test moves in the right direction away from static input--output snapshots. Yet the Tong Test still realizes its infinite task generation ``by sampling configurations (e.g., combinations of different objects, physical and social fluent states, and actions) in the continuous space of DEPSI environments'' \citep[p.~17]{peng2023tong} over a predefined ``parse graph with fluent space'' \citep[p.~17]{peng2023tong}, and still condenses the evaluation into a static profile in which ``the AGI agent yields a set of scores representing the estimated ability and value levels across individual dimensions'' \citep[p.~18]{peng2023tong} across five predefined ability and value milestones. In the terms of \S\ref{sec:recursive_teleology} and \S\ref{sec:formalization}, combinatorially sampling tasks from a fixed parse-graph fluent space and scoring agents against a pre-established five-level value hierarchy still measures coverage within a single predetermined representational regime $X_{n}$, remaining subject to the closure bound of \eqref{eq:gcond}. Evaluating open-ended intelligence requires taking the dynamic impulse of the Tong Test one step further: rather than scoring where an agent settles on a fixed capability-and-value radar plot after solving combinatorially sampled tasks, evaluation must observe the recursive movement $X_{n} \to X_{n+1}$ itself---testing whether, when confronted with problems that cannot be solved within its current parse graph and value coordinates, the system expends compute on that unsolvable frontier and reifies the resulting failed derivation traces into genuinely new representational primitives and purposes.

The central implication is accordingly:

\begin{quote}
\fbox{\parbox{0.9\textwidth}{
\centering
\textit{Open-ended intelligence involves recursively transforming ambiguity into a new semantic space rather than eliminating it.}
}}
\end{quote}

The objective of such a system is not to reduce the range of possible meanings as far as possible, but to create and maintain the conditions under which productive meanings can continue to emerge.

\bibliographystyle{plainnat}
\bibliography{references}

\end{document}